%% file: example_paper.tex
\documentclass{article}

\usepackage{microtype}
\usepackage{graphicx}
\usepackage{subcaption}
\usepackage{booktabs} 
\usepackage{paralist}
\usepackage{comment} 

\usepackage[preprint]{icml2026}

\usepackage{amsmath}
\usepackage{amssymb}
\usepackage{mathtools}
\usepackage{amsthm}

\usepackage{hyperref}

\usepackage[capitalize,noabbrev]{cleveref}

\theoremstyle{plain}
\newtheorem{theorem}{Theorem}[section]
\newtheorem{proposition}[theorem]{Proposition}
\newtheorem{lemma}[theorem]{Lemma}
\newtheorem{corollary}[theorem]{Corollary}
\theoremstyle{definition}
\newtheorem{definition}[theorem]{Definition}
\newtheorem{assumption}[theorem]{Assumption}
\theoremstyle{remark}
\newtheorem{remark}[theorem]{Remark}
\usepackage{color}
\usepackage{enumitem}

\usepackage[textsize=tiny]{todonotes}

\icmltitlerunning{Effective Frontiers: A Unification of Neural Scaling Laws}

\begin{document}

\twocolumn[
  \icmltitle{Effective Frontiers: A Unification of Neural Scaling Laws}




  \icmlsetsymbol{equal}{*}

  \begin{icmlauthorlist}
    \icmlauthor{Jiaxuan Zou}{equal,xjtu}
    \icmlauthor{Zixuan Gong}{equal,ruc}
    \icmlauthor{Ye Su}{yyy,sch1}
    \icmlauthor{Huayi Tang}{ruc}
    \icmlauthor{Yong Liu}{ruc}
  \end{icmlauthorlist}

  \icmlaffiliation{xjtu}{Department of Mathematics and Statistics, Xi'an Jiaotong University, Xi'an 710049, Shaanxi Province, China}
  \icmlaffiliation{ruc}{Department of Artificial Intelligence, Gaoling School of Artificial Intelligence, Renmin University of China, Beijing 100872, China}
    \icmlaffiliation{yyy}{University of Chinese Academy of Sciences, Beijing, 101407, China}
      \icmlaffiliation{sch1}{Shenzhen Institutes of Advanced Technology, Chinese Academy of Sciences, Shenzhen, 518055, China}
      
  \icmlcorrespondingauthor{Yong Liu}{liuyonggsai@ruc.edu.cn}

  \icmlkeywords{Machine Learning, ICML}

  \vskip 0.3in
]



\printAffiliationsAndNotice{\icmlEqualContribution}

\begin{abstract}
Neural scaling laws govern the prediction power-law improvement of test loss with respect to model capacity ($N$), datasize ($D$), and compute ($C$). However, existing theoretical explanations often rely on specific architectures or complex kernel methods, lacking intuitive universality. In this paper, we propose a unified framework that abstracts general learning tasks as the progressive coverage of patterns from a long-tail (Zipfian) distribution. We introduce the \textbf{Effective Frontier} ($k_\star$), a threshold in the pattern rank space that separates learned knowledge from the unlearned tail. We prove that reducible loss is asymptotically determined by the probability mass of the tail a resource-dependent \textbf{frontier truncation}. Based on our framework, we derive the precise scaling laws for $N$, $D$, and $C$, attributing them to \textbf{capacity}, \textbf{coverage}, and \textbf{optimization} bottlenecks, respectively. Furthermore, we unify these mechanisms via a \textbf{Max-Bottleneck} principle, demonstrating that the Kaplan and Chinchilla scaling laws are not contradictory, but equilibrium solutions to the same constrained optimization problem under different active bottlenecks.
\end{abstract}

\section{Introduction}
\label{sec:introduction}

Neural scaling laws have established themselves as the foundational empirical constants of modern machine learning. In domains ranging from language modeling \citep{brown2020languagemodelsfewshotlearners, kaplan2020scalinglawsneurallanguage} and computer vision \citep{zhai2022scalingvisiontransformers} to multimodal generation \citep{henighan2020scalinglawsautoregressivegenerative}, test loss reliably follows precise power-law trajectories with respect to key resources: model size $N$, dataset size $D$, and training compute $C$~\citep{hestness2017deeplearningscalingpredictable}. These scaling laws serve as the guiding compass for resource allocation in the training of foundation models.

However, despite the \textbf{\textit{empirical universality}} of these laws across diverse architectures and datasets, their \textbf{\textit{theoretical principles}} remain elusive. 
Currently, these laws are often treated as observational constants rather than derived mathematical necessities.
This limited understanding restricts their prediction power, most notably leaving us unable to definitively reconcile seemingly divergent regimes, such as the parameter-centric scaling prescribed by \citet{kaplan2020scalinglawsneurallanguage} (Kaplan) versus the data-centric scaling later demonstrated by \citet{hoffmann2022trainingcomputeoptimallargelanguage} (Chinchilla).
This motivates us to establish a unified framework, theoretically deriving the scaling laws across various resource constraints.



Theoretical efforts to explain scaling laws have largely fragmented into single perspectives, each capturing only part of the picture. One line of work relies on \textit{solvable surrogate models}, such as kernel methods \citep{jacot2020neuraltangentkernelconvergence, bordelon2021spectrumdependentlearningcurves, maloney2022solvablemodelneuralscaling, Bahri_2024}. While mathematically rigorous, these are often restricted to specific architectural limits and lack universality. Another perspective focuses on \textit{spectral analysis}, relating scaling to the intrinsic dimension of manifolds \citep{sharma2020neuralscalinglawdimension, Spigler_2020}. However, these continuous methods often overlook the discrete, heavy-tailed nature (e.g., Zipfian laws) inherent to natural data~\citep{pan2025understandingllmbehaviorscompression}. Beyond static spectral limits, another direction analyzes \textit{optimization dynamics} \citep{bordelon2022learningcurvessgdstructured, arous2025learningquadraticneuralnetworks, ren2025emergencescalinglawssgd, luo2025multipowerlawlosscurve, li2025functionalscalinglawskernel}, but typically assumes infinite capacity or data, failing to account for the interplay between static resource limits ($N, D$) and dynamic learning processes ($C$). Crucially, a unified framework that derives the scaling laws for all three bottlenecks ($N, D, C$) and reconciles their conflicting trade-offs is currently lacking.

In this work, we view this challenge as an analogue of Hilbert's sixth problem: the derivation of macroscopic thermodynamics from microscopic statistical mechanics~\citep{deng2025hilbertssixthproblemderivation}. We abstract away architectural details and model the task as the progressive harvesting of information from a long-tailed distribution. We propose the concept of an \textbf{Effective Frontier} ($k_\star$), a sharp boundary in the pattern rank space that separates learned knowledge from the unlearned tail~\citep{pan2025understandingllmbehaviorscompression}. We show that scaling laws are simply the manifestation of this frontier advancing into the heavy-tailed (Zipfian) data distribution under resource constraints. Our main contributions are summarized as follows:

\begin{itemize}
    \item \textbf{Unification via Effective Frontier:} 
    We propose a unified framework that conceptualizes learning as the progressive advancement of an Effective Frontier $k_\star$ in the rank space. Specifically, it reveals that the scaling laws associated with distinct resource bottlenecks, such as model capacity ($N$), datasize ($D$) and compute ($C$), are fundamentally governed by the specific rate at which this frontier advances into the heavy tail of the data distribution (Section~\ref{sec:problem_setup}$\sim$\ref{sec:frontier_principle}).
    \item \textbf{Derivation of Scaling Laws:} We analytically derive the precise scaling exponents for $N$, $D$, and $C$ by characterizing their distinct bottleneck mechanisms: static bottlenecks governed by memory capacity (Section~\ref{subsec:model_scaling}) and statistical coverage (Section~\ref{subsec:data_scaling}), dynamic constraints arising from the optimization dynamics and implicit bias of gradient descent (Section~\ref{sec:compute_scaling}).
    This derivation serves as a unified basis connecting the statistical properties of data to diverse physical constraints.
    \item \textbf{Reconciling Optimal Frontiers:} By formulating the joint loss as a ``Max-Bottleneck'' problem, $\Delta L \asymp \max(\varepsilon_N, \varepsilon_D, \varepsilon_\tau)$, we resolve the apparent conflict between the Kaplan and Chinchilla scaling laws~\citep{kaplan2020scalinglawsneurallanguage, hoffmann2022trainingcomputeoptimallargelanguage}. We show they are not contradictory, but represent equilibrium solutions to the same constrained optimization problem under different active bottlenecks (Section~\ref{sec:joint_scaling}).
\end{itemize}


\section{Problem Setup}
\label{sec:problem_setup}

\textbf{Notations.}
Throughout this paper, we use notations $o$ and $\asymp$ to depict the asymptotic behavior as $R \to \infty$. Specifically, $g(R)=o(f(R))$ indicates that $\lim_{R \to \infty} g(R)/f(R) = 0$. We simply write $g(R)=o_R(1)$ when $f(R)=1$, which denotes a quantity that tends to zero as $R \to \infty$. Besides, $g(R) \asymp f(R)$ means that there exist two positive constants $c_1, c_2$ such that $c_1 f(R) \leq g(R) \leq c_2 f(R)$, and $f(R) \propto g(R)$ denotes that there exists positive constant $c$, such that $f(R) = cg(R)$. All these constants are independent of $R$.


\textbf{Additive Decomposition.}
Denote by $L$ the total test loss, the \textit{reducible loss} is defined as $\Delta L := L - E$, where $E$ is the irreducible entropy (Bayes error) of the data distribution. Intuitively, $\Delta L$ represents the performance gap that can be bridged by learning resources. To derive universal scaling laws, a crucial and necessary step is to mathematically model the reducible loss. On one hand, we aim to incorporate the main factors affecting the reducible loss. On the other hand, to facilitate theoretical analysis, we hope the relationships among them remain as simple as possible. Therefore, in this paper, we propose that the learnable structure of the data distribution can be decomposed into a set of discrete and learnable units, which we term as atomic patterns. In other words, we assume that the complexity of any real-world predictive task arises from the composition of countably infinite underlying structures. Motivated by this, we introduce the following definition.

\begin{assumption}[Additive Pattern Model]
\label{assump:decomposition}
Suppose that the learning task factorizes into independent sub-problems. Consequently, the scalar reducible loss $\Delta L$ is exactly represented as the expectation of pattern-wise residuals:
\begin{equation}
    \label{eq:loss_decomposition}
    \Delta L = \sum_{k=1}^{\infty} p_k q_k,
\end{equation}
where $k$ indexes the distinct atomic patterns. $p_k \ge 0$ with $\sum p_k = 1$ represents the statistical frequency of pattern $k$, and $q_k \in [0, 1]$ is the normalized residual (excess risk) for pattern $k$. Here, $q_k=1$ implies the pattern is completely unlearned, while $q_k \to 0$ implies mastery.
\end{assumption}

Strictly speaking, weak inter-pattern correlations may exist in finite regimes. Nevertheless, Equation~\eqref{eq:loss_decomposition} can be viewed as the asymptotic limit applicable to common cases, where such correlations are neglected. More importantly, Equation~\eqref{eq:loss_decomposition} linearizes the macroscopic loss into independent microscopic components, allowing us to analyze the behavior of $q_k$ under different resource constraints separately.

\textbf{Zipf Tail.}
The observation of scaling laws necessitates a specific structural form for the weights $\{p_k\}_{k=1}^\infty$. An exponential decay would result in nearly immediate saturation of learning, which is not consistent with observations in practices. Rather, the power-law scaling of the loss originates in the fundamental power-law distribution of the data. This relationship is formalized by the Zipfian distribution.

\begin{assumption}[Zipfian Distribution]
\label{assump:zipf}
Suppose that the patterns are sorted by importance (i.e., $p_1 \ge p_2 \ge \dots$) and they follow the Zipfian distribution given by
\begin{equation}
    \label{eq:zipf_law}
    p_k = \frac{1}{Z(\alpha)} k^{-\alpha},
\end{equation}
where $\alpha>1$ is the decay exponent and $Z(\alpha)$ is the normalization constant defined as $Z(\alpha) = \sum_{k=1}^\infty k^{-\alpha}$.
\end{assumption}

\paragraph{Tail Behavior.}
The exponent $\alpha$ characterizes the heaviness of the tail. Specifically, in the limit $\alpha \to 1^+$, the distribution approaches an extremely heavy-tailed regime where the aggregate probability mass contained within infinitely many rare patterns remains significant. Thus, this inherent statistical structure underpins our unified framework, where the observed scaling phenomenon is metaphorically the process of “harvesting” this dispersed mass by systematically pushing a coverage frontier deeper into the tail.

\begin{figure*}[t]
   \centering
   \includegraphics[width=0.98\textwidth]{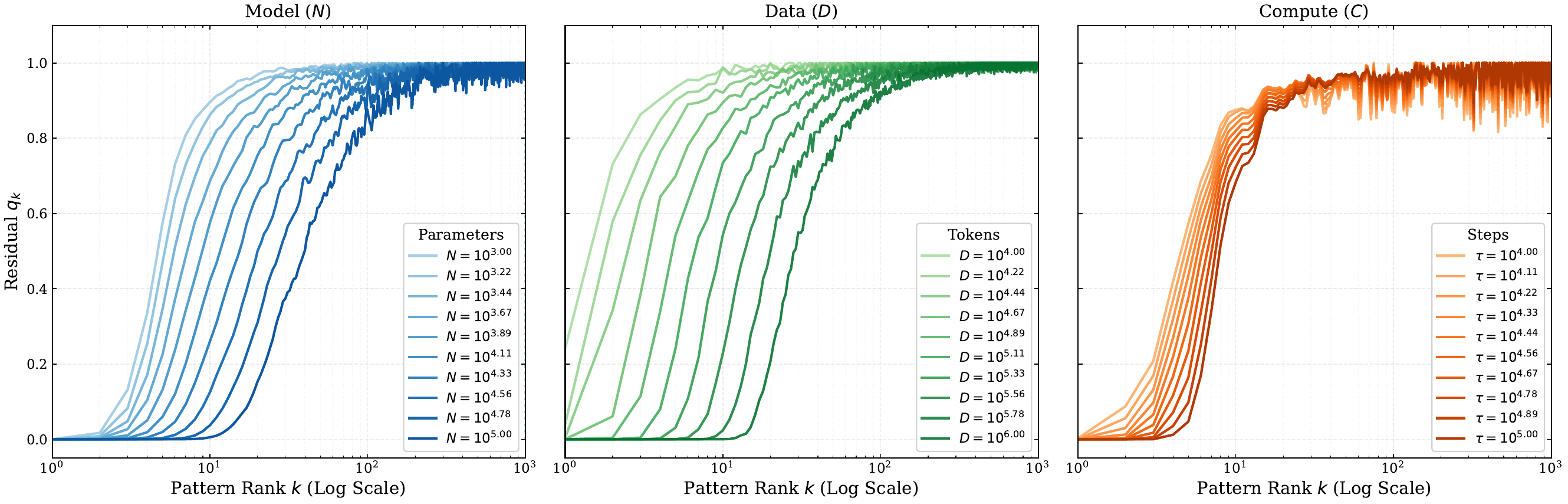}
     \caption{\textbf{The Effective Frontier in Rank Space.} The residual profile $q_k$ vs. pattern rank $k$ (log scale) under varying constraints: (Left) Model Capacity $N$, (Middle) Dataset Size $D$, and (Right) Compute $\tau$ (noting that $C \propto \tau$ for a fixed model configuration). In all cases, increasing resources pushes a sharp coverage frontier $k_\star$ deeper into the Zipfian tail ($\alpha=1.5$). This validates the geometric abstraction (Definition~\ref{def:effective_frontier}) where learning is viewed as a progressive coverage process.}
\label{fig:rank_space_exp} 
\end{figure*}
\section{Effective Frontier}
\label{sec:frontier_principle}
In this section, we bridge the atomic pattern decomposition in Assumption~\ref{assump:decomposition} to a resource-dependent geometric picture in rank space.

Let $R$ denote a scalar resource budget representing model capacity $N$, datasize $D$, or compute $\tau$. The training outcome under budget $R$ is characterized by a \emph{pattern-wise residual profile} $\{q_k(R)\}_{k\ge 1}$, where $q_k(R)\in[0,1]$ quantifies the normalized excess risk of pattern $k$.
We regard $q_k(\cdot)$ as a function of resources $R$. Naturally, additional resources should not degrade performance, implying that $q_k(R)$ is monotonically non-increasing in $R$.

To translate this profile into a one-dimensional geometry, we invoke the greedy learning bias, which suggests that frequently occurring patterns are learned preferentially~\citep{gong2025disentanglingfeaturestructuremathematically}.
\begin{assumption}[Greedy Learning Bias]
    \label{assump:rank_monotone}
    The residual profile is monotonically non-decreasing with respect to the frequency rank $k$, \textit{i.e.}, for all $k \geq 1$, $q_k(R) \le q_{k+1}(R).$
\end{assumption}

With a fixed tolerance $\delta \in (0, 1/2)$, we define the transition boundaries of the learned region as the last \textit{mostly learned} index $k_{-}(R) := \sup\{k: q_k \le \delta\}$ and the first \textit{mostly unlearned} index $k_{+}(R) := \inf\{k: q_k \ge 1-\delta\}$.
Assumption~\ref{assump:rank_monotone} ensures $k_{-}(R) \le k_{+}(R)$, preventing learned and unlearned patterns from interleaving.
We now use these boundaries to formalize the concept of a effective frontier.

\begin{definition}[Effective Frontier]\label{def:effective_frontier}
    We define $k_\star(R)$ as a sharp effective frontier if it asymptotically characterizes the transition boundaries:
    \begin{equation*}
        \label{eq:frontier_sharp}
        \begin{aligned}
        k_{-}(R) &= \bigl(1-o_R(1)\bigr)\,k_\star(R),\\
        k_{+}(R) &= \bigl(1+o_R(1)\bigr)\,k_\star(R).
        \end{aligned}
    \end{equation*}
\end{definition}

Definition~\ref{def:effective_frontier} implies that the relative width of the transition window vanishes asymptotically, \textit{i.e.}, $(k_+ - k_-)/k_\star \to 0$. 
Intuitively, $k_\star(R)$ represents the unique macroscopic scale where learning effectively ceases (see Figure~\ref{fig:rank_space_exp}). 
This justifies approximating the residual profile $q_k(R)$ as a macroscopic step function: $q_k \approx 0$ for $k \lesssim k_\star$ and $q_k \approx 1$ for $k \gtrsim k_\star$.

As illustrated in Figure~\ref{fig:rank_space}, by treating the learning process as a discrete cutoff in rank space, the reducible loss (Assumption~\ref{assump:decomposition}) reduces to the probability mass of the unlearned tail: $\Delta L(R) \asymp \sum_{k > k_\star(R)} p_k.$
This geometric reduction allows us to derive scaling laws solely by analyzing the tail behavior of the pattern distribution.
Building upon this, we present the \textit{Universal Scaling Principle}.

\begin{theorem}[Universal Scaling Principle]
    \label{thm:tail_sum}
    Under Assumption~\ref{assump:decomposition}$\sim$\ref{assump:zipf}, if a resource $R$ induces a effective frontier $k_\star(R)$ (Definition~\ref{def:effective_frontier}), the reducible loss scales as:
    \begin{equation*}
        \label{eq:tail_sum_integral}
        \Delta L(R) \asymp k_\star(R)^{-(\alpha-1)}.
    \end{equation*}
\end{theorem}

Theorem~\ref{thm:tail_sum} reduces the derivation of any concrete scaling law to a purely geometric problem: determining the \textbf{resource-to-frontier map} $R \mapsto k_\star(R)$ for the specific bottleneck (detailed proof in Appendix~\ref{app:proof_sec4}).
We identify three distinct regimes:
\begin{itemize}[leftmargin=*, itemsep=2pt, parsep=0pt]
    \item Model Scaling (model capacity $N$): capacity-limited frontier $k_\star(N)$.
    \item Data Scaling (datasize $D$): coverage-limited frontier $k_\star(D)$.
    \item Compute Scaling (compute budget $C$ or optimization steps $\tau$): optimization-limited frontier $k_\star(\tau)$, where $C \propto \tau$ for a fixed model configuration.
\end{itemize}


\begin{figure}[t!]
  \centering
  \includegraphics[width=0.98\linewidth]{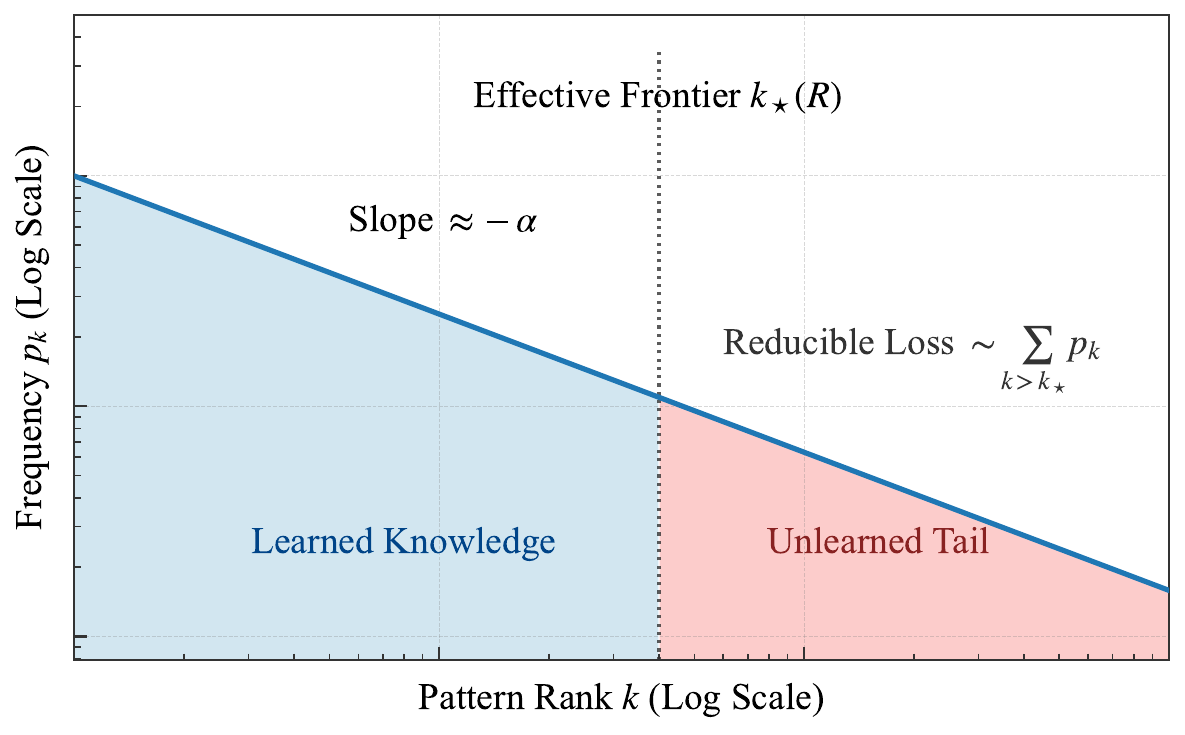}
  \caption{Effective frontier $k_\star(R)$ in rank space. Under Zipf frequencies $p_k\propto k^{-\alpha}$, resources $R$ (e.g., $N,D,\tau$) induce a cutoff: learned patterns ($k\le k_\star$) vs.\ unlearned tail ($k>k_\star$). The reducible loss is dominated by the tail sum $\sum_{k>k_\star(R)} p_k$, which decreases as $k_\star(R)$ shifts right with increasing resources.}
  \label{fig:rank_space}
\end{figure}

\section{Model Scaling: The Capacity Frontier}
\label{subsec:model_scaling}

In this section, we derive the Model Scaling Law.
We consider the regime where datasize $D$ and compute $C$ are abundant, rendering model capacity ($N$) the dominant bottleneck. 
In this limit, the effective capacity frontier $k_\star(N)$ and reducible loss are constrained by the model's \textit{Memory Capacity}, which represents the fundamental ability to store and resolve complex patterns from the heavy-tailed distribution.

To characterize the mapping $N \mapsto k_\star(N)$, we ground our analysis in the geometric concept of \textit{Effective Degrees of Freedom}.
Geometrically, learning $k$ orthogonal patterns imposes $k$ independent constraints. Resolving these constraints requires sufficient degrees of freedom, which are provided quantitatively by the parameter budget $N$.
This intuition is corroborated by recent theoretical and empirical studies, establishing that a deep network's capacity to perfectly fit random patterns scales linearly with the number of parameters  $N$~\citep{vershynin2020memory,zhang2021understanding}.


Building upon this understanding, we formalize the relationship between model size and the learnable frontier.
\begin{assumption}[Capacity Scaling] 
    \label{assump:capacity}
    The effective capacity frontier scales with the parameter count $N$ as $k_\star(N) \asymp N^{\gamma},$ where $\gamma \in (0, 1]$ is the architectural efficiency factor.
\end{assumption}
\begin{remark}
For an ideal linear model learning orthogonal features, capacity is strictly proportional to parameters ($k_\star \asymp N$ with $\gamma=1$).
In deep non-linear networks, this relationship is supported by the theory of \textbf{Memory Capacity}, which proves that the number of arbitrary patterns a ReLU network can perfectly interpolate scales linearly with its parameter count $N$ \cite{vershynin2020memory,yun2019small}. 
While practical factors such as architectural redundancy, parameter sharing or gradient-based learning inefficiencies may lead to sub-linear scaling ($\gamma < 1$) \cite{sharma2020neuralscalinglawdimension}.
Nevertheless, the fundamental dependency remains rooted in the above understandings, preserving the power-law form.
\end{remark}

In the following, we present the Model Scaling Law.
\begin{proposition}[Model Scaling Law]
\label{prop:model_scaling}
    Under Assumptions~\ref{assump:decomposition}$\sim$\ref{assump:zipf} and Assumption \ref{assump:capacity}, the reducible loss scales with model parameter $N$ as:
    \begin{equation*}
        \boxed{\Delta L(N) \asymp N^{-\gamma(\alpha-1)}.}
    \end{equation*}
\end{proposition}
This proposition~\ref{prop:model_scaling} presents the asymptotic scaling limit constrained by model capacity.
Crucially, this result structurally disentangles the empirical exponent into two components: (a) the term $(\alpha-1)$ represents the \textit{Data Structure}, governed solely by the heavy-tailedness of the distribution; (b) the factor $\gamma$ captures the \textit{Architectural Efficiency}, quantifying the specific architecture's capability to utilize parameters for pattern learning.

\section{Data Scaling: The Coverage Frontier}
\label{subsec:data_scaling}
In this section, we derive the Data Scaling Law.
Specifically, we consider the regime where model capacity and compute are abundant, leaving finite data coverage as the dominant bottleneck.
We adopt a three-step theoretical approach: 
\textit{(i)} define a coverage-induced probabilistic residual proxy at the pattern level;
\textit{ (ii)} lift it to a dataset-level reducible loss proxy using the additive decomposition; 
and \textit{(iii)} derive the induced effective frontier in the rank space.

Let the training set contain $D$ samples drawn \textit{i.i.d.} from the underlying pattern distribution $\{p_k\}_{k\ge 1}$. 
For each pattern $k$ with frequency $p_k$, its occurrence count follows a binomial distribution $X_k \sim \mathrm{Binomial}(D,p_k)$.
Motivated by the observation that no learning rule can reliably reduce the residual error of a pattern physically absent from the training set, we introduce the following residual proxy as step \textit{(i)}.

\begin{definition}[Residual Proxy]
    \label{def:residual_proxy}
    We define the coverage-induced residual proxy of pattern $k$ as the probability of non-observation, \textit{i.e.}, $q_k(D) \triangleq \Pr[X_k=0] = (1-p_k)^D$. 
\end{definition}
Definition~\ref{def:residual_proxy} implicitly assumes that a single occurrence is sufficient for learning in the infinite-compute regime.We extend this formulation in Definition~\ref{def:m_residual_proxy} by introducing a generalized threshold $m$, requiring at least $m$ occurrences for effective learning; in finite-data settings this naturally connects to multi-epoch data reuse. \citep{yan2025largerdatasetsrepeatedmore}

Proceeding to step \textit{(ii)} and \textit{(iii)}, we aggregate these pattern-level residuals to derive the dataset-level scaling behavior.
\begin{theorem}[Data Scaling Law]
    \label{thm:data_scaling}
    Under Assumption~\ref{assump:decomposition}$\sim$\ref{assump:zipf} and Definition~\ref{def:residual_proxy}, the effective coverage frontier satisfies
    \begin{equation*}
        k_\star(D) \asymp D^{1/\alpha}.
    \end{equation*}
    Consequently, the reducible loss scales with datasize $D$ as
    \begin{equation*}
        \boxed{
        \Delta L(D) \asymp D^{-\frac{\alpha-1}{\alpha}}.}
        \end{equation*}
\end{theorem}
Theorem~\ref{thm:data_scaling} presents the data scaling behavior, explicitly quantifying how the tail heaviness $\alpha$ governs the sample efficiency through the exponent $\alpha_D = (\alpha-1)/\alpha$.
In the heavy-tailed limit ($\alpha \to 1^+$), this exponent vanishes ($\alpha_D \to 0$), reflecting the inherent inefficiency of \textit{i.i.d.} sampling in covering the long tail of rare patterns. 
Conversely, a larger $\alpha$ (lighter tail) yields a larger $\alpha_D$, accelerating the rate at which additional data translates into effective coverage and loss reduction (see Figure~\ref{fig:data_efficiency}). Detailed proof is deferred to Appendix~\ref{app:proof_thm_data}.

While Definition~\ref{def:residual_proxy} provides a simple residual proxy, realistic learning often requires repeated exposure to distinguish signal from noise. To capture this, we introduce a generalized proxy requiring a minimum threshold of occurrences.
\begin{definition}[Generalized Residual Proxy]\label{def:m_residual_proxy}
    We define the generalized coverage-induced residual proxy of pattern $k$ as the probability that the pattern occurrence count $X_k$ falls below a threshold $m>1$, \textit{i.e.}, $q_k(D)\triangleq \Pr[X_k<m]$. 
\end{definition}

\begin{corollary}
\label{cor:data_scaling_m}
    Under Assumptions~\ref{assump:decomposition}$\sim$\ref{assump:zipf} and Definition~\ref{def:m_residual_proxy}, the effective coverage frontier satisfies
    \begin{equation*}
        k_\star(D) \;\asymp\; {\frac{D}{Zm}}^{1/\alpha}.
    \end{equation*}
    Consequently, the reducible loss scales with datasize $D$ as
    \begin{equation*}
        \boxed{\Delta L(D) \asymp D^{-\frac{\alpha-1}{\alpha}}.}
    \end{equation*}
\end{corollary}
Corollary~\ref{cor:data_scaling_m} presents the universality of the data scaling law, demonstrating that the asymptotic exponent remains invariant to the minimum occurrence threshold $m$.
Detailed proof is deferred to Appendix~\ref{app:proof_cor_m}.

\begin{figure}[t]
    \centering
    \includegraphics[width=1.0\linewidth]{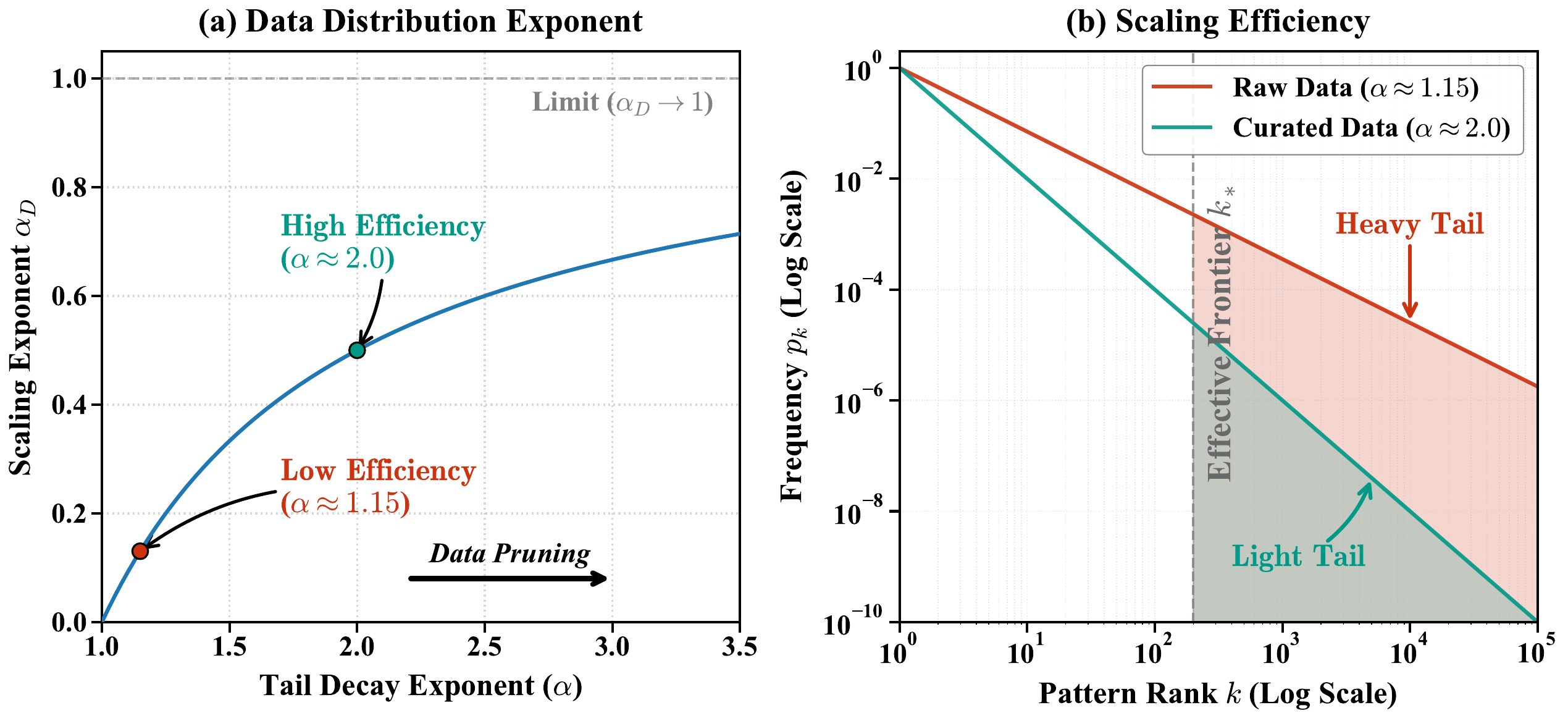}
\caption{\textbf{Tail Heaviness Governs Data Scaling Efficiency.}
\textbf{(a)} The data scaling exponent $\alpha_D$ is analytically determined by the tail index $\alpha$ (Theorem~\ref{thm:data_scaling}). Lighter tails (larger $\alpha$, e.g., via data pruning) yield strictly higher sample efficiency.
\textbf{(b)} Geometrically, for a fixed effective frontier $k_*$, heavy-tailed distributions (Red) retain significantly more unlearned probability mass (shaded tail) than curated distributions (Teal), resulting in slower loss reduction.}
\label{fig:data_scaling_efficiency}
\label{fig:data_efficiency}
\end{figure}

\section{Compute Scaling: The Optimization Frontier}
\label{sec:compute_scaling}
In this section, we derive the Compute Scaling Law. 
While model capacity ($N$) and datasize ($D$) impose static limits, compute ($C$) introduces a dynamic constraint on the learning process. Even with infinite capacity and data, reducible loss remains non-zero if the optimization has not converged for rare \textit{tail} patterns.
To establish this, we first derive the exact scaling exponent for Gradient Descent (GD) from optimization dynamics analysis (Section~\ref{sec:sgd_dynamics}).
Subsequently, we extend this result by modeling the residual evolution through a \textit{Self-Similar Scaling Kernel}, formalizing a general scaling law independent of specific optimizer (Section~\ref{sec:general_compute_scaling}).
\subsection{Optimization Dynamics and Scaling Law for GD}\label{sec:sgd_dynamics}
We adopt GD as a canonical example to ground our analysis of compute scaling. Specifically, we identify a dynamic effective frontier induced by the optimization process and derive the resulting compute scaling law in the GD regime. 
Before stating the formal theorem, we first establish the necessary assumptions and lemmas. 
We consider decoupled optimization trajectories of independent patterns (Assumption~\ref{assump:decomposition}).
Let $L_k(\theta)$ denote the loss component associated with pattern $k$, and $\nabla L_k(\theta)$ its gradient with respect to parameters $\theta$.
\begin{assumption}[Orthogonality and Polyak-Lojasiewicz (PL) Condition]
    \label{assump:PL_condition}
    We assume that the loss landscape satisfies the following geometric conditions:
    \begin{enumerate}[label=\textit{(\alph*)}, leftmargin=*]
        \item \textit{$\epsilon$-Orthogonality:} For any pair of distinct patterns $j \neq k$, the inner product of their gradients is bounded by a small constant $\epsilon \ge 0$: $|\langle \nabla L_j(\theta), \nabla L_k(\theta) \rangle| \le \epsilon$.
        \item \textit{Pattern-wise PL Condition:} For each pattern $k$, the gradient norm is lower-bounded by its residual $q_k$: $\|\nabla L_k(\theta)\|^2 \ge 2\lambda_k q_k(\theta)$, where $\lambda_k > 0$ denotes the pattern-specific PL constant.
        \item \textit{Pattern-wise Smoothness:} Loss component $L_k$ is $\beta_k$-smooth on the trajectory region, i.e. $\|\nabla L_k(\theta)-\nabla L_k(\theta')\|\le \beta_k \|\theta-\theta'\|$.

    \end{enumerate}
\end{assumption}
Assumption~\ref{assump:PL_condition} formalizes the geometry properties of the loss landscape. 
Condition (a) ensures that gradients are nearly orthogonal in the parameter space, implying that the learning trajectories of distinct patterns are effectively decoupled up to an $\epsilon$-error.
Condition (b) posits that each loss component satisfies the PL inequality, a standard assumption in non-convex optimization analysis~\citep{polyak1963gradient,lojasiewicz1963topological,karimi2016linear}. 
Here, $\lambda_k$ represents the pattern-specific PL constant, corresponding to the minimum eigenvalue of the block-diagonal Hessian within the local quadratic regime. 

Building on these conditions, we formalize the evolution of pattern-wise residual $q_k$ in the following Lemma~\ref{lemma:q_dynamic}.
\begin{lemma}[Iterative Residual Dynamics]
    \label{lemma:q_dynamic}
    Let $I_{t,k} \in \{0, 1\}$ be the indicator variable such that $I_{t,k}=1$ if pattern $k$ is sampled at step $t$, and $I_{t,k}=0$ otherwise. Under Assumption~\ref{assump:PL_condition}, the residual $q_k(t)$ evolves as
    \[
    q_k(t+1) = \begin{cases} (1-\eta \lambda_k)\,q_k(t) & \text{if } I_{t,k}=1, \\ q_k(t) & \text{if } I_{t,k}=0, \end{cases}
    \]
    where $\eta$ denotes the (constant) learning rate and $\lambda_k$ is the pattern-specific PL constant.
\end{lemma}
Lemma~\ref{lemma:q_dynamic} indicates that the residual contracts geometrically only upon sampling.
We term $\lambda_k$ the \textit{effective error correction coefficient}, as it determines the per-step reduction rate.

\begin{remark}
More general learning-rate schedules $\eta=\eta(t)$ lead to the same “only contracts upon sampling” structure, but with a time-varying contraction factor; recent work studies functional scaling laws that explicitly incorporate learning-rate schedules. \citep{li2025functionalscalinglawskernel}
\end{remark}

Detailed proof is deferred to Appendix \ref{app:sgd_micro}.

Furthermore, we posit that $\lambda_k$ is not uniform across patterns and formalize this inhomogeneity based on theoretical insights under deep linear network, where the convergence rate intrinsically scales as a power of pattern frequency $p_k$.
\begin{assumption}[Inhomogeneity]
    \label{assump:spectral_power}
    The effective error correction coefficient $\lambda_k$ follows the power law~$\lambda_k \propto p_k^{\beta-1}$, where $\beta>0$ is the implicit bias exponent.
\end{assumption}
\textbf{Theoretical Justifications.}\quad
Assumption~\ref{assump:spectral_power} formalizes the inductive bias that frequent patterns are learned preferentially, manifesting as inhomogeneous dynamics. 
We derive this power law via gradient flow analysis on Deep Linear Networks (DLN) in Appendix~\ref{app:derivation_assumptions}. Specifically, we demonstrate that the bias exponent $\beta$ is determined by the network depth $L$ and the learnability exponent $\zeta$ (Assumption~\ref{assump:learnability}), obeying $\beta = 2 + \zeta(2 - 2/L)$.
Crucially, this result implies $\beta \ge 2$, placing the optimization in the \textit{rich} or \textit{feature learning} regime, where depth introduces a non-linear selectivity that actively accelerates the learning of dominant high-frequency features.
Our theoretical analysis further validates Lemma~\ref{lemma:q_dynamic} using the specific DLN architecture, without relying on PL condition in Assumption~\ref{assump:PL_condition}.

\begin{lemma}[Asymptotic Residual Dynamics]
    \label{lemma:q_dynamic_exp}
    Under Lemma~\ref{lemma:q_dynamic} and Assumption~\ref{assump:spectral_power}, with a large training step $\tau$ and a constant $c>0$, the residual $q_k(\tau)$ converges in probability to the exponential form:
    \begin{equation}
        \label{eq:exponential_decay}
        q_k(\tau) \asymp \exp\!\big(-c\,\tau\,p_k^{\beta}\big).
    \end{equation}
\end{lemma}

\begin{theorem}[Compute Scaling Law for GD]
    \label{thm:sgd_scaling}
    Under Assumptions~\ref{assump:decomposition}$\sim$\ref{assump:zipf} and Lemma~\ref{lemma:q_dynamic_exp}, the dynamic effective frontier $k_\star(\tau)$ satisfies
    \begin{equation*}
        \label{eq:frontier_scaling}
        k_\star(\tau) \asymp \tau^{\frac{1}{\alpha\beta}}.
    \end{equation*}
    Consequently, the reducible loss scales as
    \begin{equation*}
        \label{eq:sgd_scaling_result}
        \boxed{\Delta L(\tau) \asymp \tau^{-\frac{\alpha-1}{\alpha\beta}}.}
    \end{equation*}
\end{theorem}
Theorem~\ref{thm:sgd_scaling} bridges the microscopic optimization behavior in Lemma~\ref{lemma:q_dynamic_exp} with macroscopic compute scaling law.
Specifically, it characterizes the optimization dynamics where the effective frontier $k_\star(\tau)$ progressively advances into the tail.
This mechanism reveals that the reducible loss is dominated by the cumulative unlearned mass beyond this moving frontier.
Consequently, the asymptotic compute scaling is governed by the interplay between the data's tail heaviness $\alpha$ and the model's inductive bias $\beta$.
Detailed proofs are deferred to Appendix~\ref{app:sgd_micro}.

\subsection{Extension to General Compute Scaling Law}\label{sec:general_compute_scaling}
The scaling law established in Theorem~\ref{thm:sgd_scaling} builds upon the exponential residual decay that arises specifically from the GD dynamics (Lemma~\ref{lemma:q_dynamic_exp}). 
Given that empirical scaling laws demonstrate remarkable universality across diverse optimizers, we now extend Equation~\eqref{eq:exponential_decay} to a general \textit{Self-Similar Scaling Kernel}.
\begin{assumption}[Self-Similar Scaling Kernel] 
    \label{assump:q_dynamic_g}
    Assume that the residual evolution follows a self-similar form governed by a monotonic scaling kernel $g: [0, \infty) \to [0, 1]$:
    \begin{equation*}
        \label{eq:general_kernel}
        q_k(\tau) \asymp g\!\big(c\,\tau\,p_k^\beta\big),
    \end{equation*}
    where $c>0$ is a constant. The kernel $g$ satisfies the boundary condition $g(0)=1$ and the integrability condition $\int_0^\infty u^{s-1}g(u)du < \infty$.
\end{assumption}
Assumption~\ref{assump:q_dynamic_g} extends the exponential decay under GD to arbitrary optimizers. While the specific profile $g(\cdot)$ may vary (\textit{e.g.}, exponential or polynomial), the fundamental self-similar property persists: the learning curves of distinct patterns collapse onto a single master curve $g(\cdot)$ when parameterized by the effective time $\tau p_k^\beta$.

We next formalize the general compute scaling law.
\begin{theorem}[General Compute Scaling Law]
\label{thm:compute_scaling}
Under Assumptions~\ref{assump:decomposition}$\sim$\ref{assump:zipf}, and Assumption~\ref{assump:q_dynamic_g}, with a large training step $\tau$, the reducible loss scales as
\begin{equation*}
    \label{eq:universal_scaling_formula}
    \boxed{\Delta L(\tau) \asymp \mathcal{K} \cdot \tau^{-s},}
\end{equation*}
where $s=(\alpha-1)/\alpha\beta$ and the pre-factor $\mathcal{K}$ is given by $\mathcal{K} = (\alpha\beta)^{-1} Z^{-1/\alpha} c^{-s} [ \int_{0}^{\infty} u^{s-1}g(u)\,du ]$.
\end{theorem}
Theorem~\ref{thm:compute_scaling} clarifies the minimal conditions underlying compute scaling laws, explicitly decoupling between the general exponent $s=(\alpha-1)/\alpha\beta$ and the kernel-dependent pre-factor.
Crucially, we identify that the scaling rate is intrinsic to the data's tail heaviness $\alpha$ and inductive bias $\beta$, while the shape of $g$ only affects coefficient $\mathcal{K}$. 
Detailed proof is deferred to Appendix~\ref{app:universality}.

\section{Joint Scaling and Unified Frontiers}\label{sec:joint_scaling}

In this section, we propose a \textbf{Composition Law} based on asymptotic dominance to resolve the conflict between the scaling laws of \citet{kaplan2020scalinglawsneurallanguage} and \citet{hoffmann2022trainingcomputeoptimallargelanguage}.

\begin{proposition}[Max-Bottleneck Principle]
\label{prop:max_bottleneck}
The joint reducible loss is asymptotically dominated by the tightest constraint among capacity ($N$), datasize ($D$), and optimization steps ($\tau$):
\begin{equation*}
    \label{eq:max_structure}
    \Delta L(N, D, \tau) \asymp \max\left( \varepsilon_N(N), \; \varepsilon_D(D), \; \varepsilon_\tau(\tau) \right).
\end{equation*}
\end{proposition}
Although empirical laws often appear additive, the asymptotic scaling is dictated by the dominant term. This structure implies a Turnover Point $\tau_\star$ where the dynamic optimization bottleneck $\varepsilon_\tau$ intersects the static limit determined by model or data size (see Appendix \ref{app:turnover_analysis}).


\paragraph{Unified Optimal Scaling.}

We can derive Optimal Scaling Laws not as empirical fits, but as solutions to a constrained optimization problem. The optimal frontier is the set of configurations $(N, D, \tau)$ that minimize the joint loss given a compute budget $C$:
\begin{equation}
    \min_{N, D, \tau} \max \left( A N^{-\alpha_N}, B D^{-\alpha_D}, G \tau^{-\alpha_\tau} \right).
    \label{eq:joint_optimization}
\end{equation}

Minimizing Equation \eqref{eq:joint_optimization} under a multiplicative budget $C$ requires balancing the active bottlenecks. This framework recovers two distinct regimes as equilibrium solutions (derivations in Appendix \ref{app:optimal_derivations}):

\textbf{Regime A: Data-Abundant (\citet{kaplan2020scalinglawsneurallanguage} ).} In the data-abundant limit ($D \to \infty$), the bottleneck is between capacity and optimization steps ($N$ vs. $\tau$). Balancing $\varepsilon_N \asymp \varepsilon_\tau$ yields the parameter-centric scaling $N_{\text{opt}} \propto C^{\frac{\alpha_\tau}{\alpha_N + \alpha_\tau}}$. Crucially, this exponent depends on the \textit{optimization bias} $\beta$.

\textbf{Regime B: Data-Constrained (\citet{hoffmann2022trainingcomputeoptimallargelanguage}).} In the fixed-data limit ($D \propto C/N$), the bottleneck is between capacity and data coverage ($N$ vs. $D$). Balancing $\varepsilon_N \asymp \varepsilon_D$ yields the balanced scaling $N_{\text{opt}} \propto C^{\frac{\alpha_D}{\alpha_N + \alpha_D}}$, where the exponent is governed by the \textit{data tail thickness} $\alpha$.

In summary, our framework reveals that the Kaplan and Chinchilla laws are thus valid solutions to Equation~\eqref{eq:joint_optimization} under different active constraints. The optimal scaling exponent is a topological invariant determined by the interplay between data structure ($\alpha$) and optimization dynamics ($\beta$).

\section{Experiments}
\label{sec:experiments}

In this section, we empirically validate our unified framework using the controlled synthetic environment described in Appendix~\ref{app:exp_details}. Our goal is to verify three central claims: (1) Resource constraints manifest as a sharp effective frontier $k_\star$ in rank space; (2) The frontier advances according to specific geometric power laws derived in Theorems~\ref{prop:model_scaling}, \ref{thm:data_scaling}, and \ref{thm:sgd_scaling}; (3) The resulting reducible loss scaling exponents are strictly determined by the interplay between data tail heaviness ($\alpha$) and the derived frontier scaling.

\subsection{Visualization of the Effective Frontier}
Figure~\ref{fig:rank_space_exp} visualizes the pattern-wise residual $q_k$ against the frequency rank $k$ for a fixed data distribution ($\alpha=1.5$). The results revealed that a distinct phase transition that separates ``learned'' patterns ($q_k \approx 0$) from the ``unlearned'' tail ($q_k \approx 1$) via a narrow region. As resources increased, whether model capacity $N$ (left), datasize $D$ (middle), or computer $\tau$ (right), this frontier systematically translated to the right while maintaining its shape. These empirical results validate our theoretical approximation of learning as a step function (Definition~\ref{def:effective_frontier}) and confirm that diverse physical bottlenecks can be universally abstracted as the progressive advancement of the Effective Frontier $k_\star$ into the heavy-tailed distribution.

\begin{figure*}[t]
  \centering
  \includegraphics[width=0.98\textwidth]{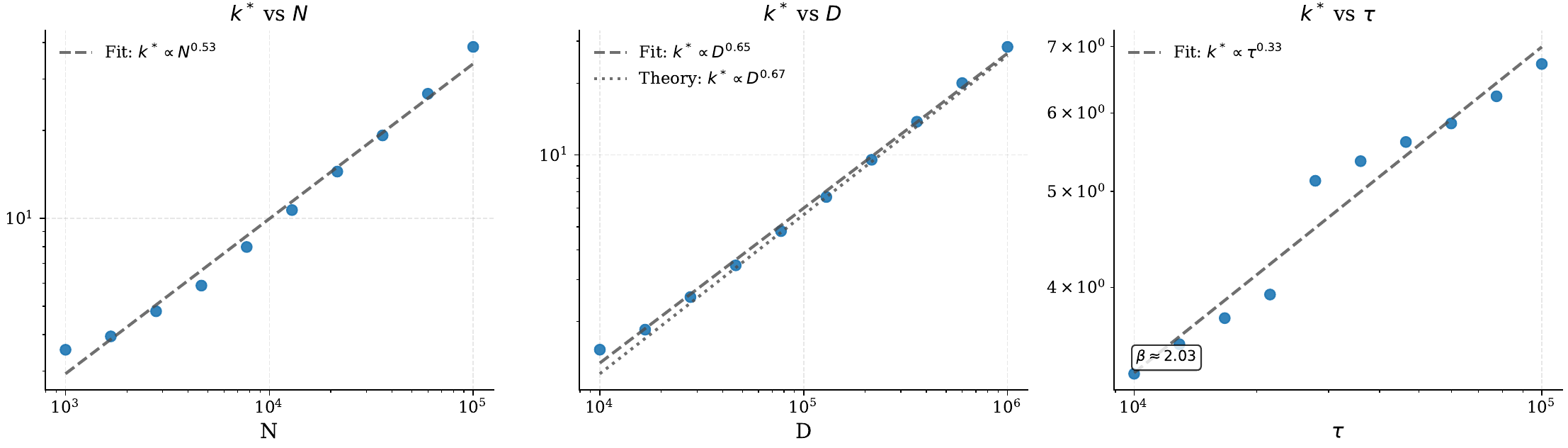}
\caption{\textbf{Geometric Scaling of the Frontier.} We extract $k_\star$ (at threshold $\delta=0.5$) as a function of resources. (Left) Capacity frontier $k_\star \propto N^{0.53}$. (Middle) Coverage frontier $k_\star \propto D$. Crucially, the theoretical prediction $D^{1/\alpha}$ (thin dash) perfectly matches the empirical fit (thick dash) for $\alpha=1.5$, confirming Theorem~\ref{thm:data_scaling}. (Right) Optimization frontier $k_\star \propto \tau$. Since total compute $C \propto \tau$ for a fixed model, this reflects the compute scaling law. The slope allows us to estimate the optimization bias $\beta \approx 2.03$.}
\label{fig:kstar_scaling}
\end{figure*}

\begin{figure*}[t]
  \centering
  \includegraphics[width=0.98\textwidth]{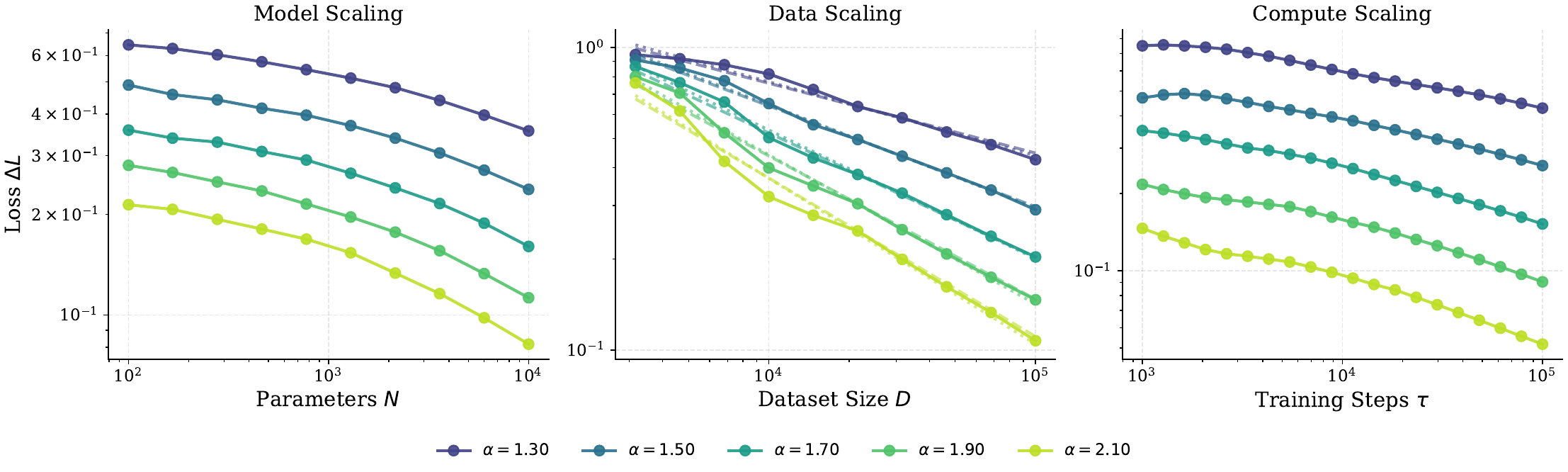}
\caption{\textbf{Universal Scaling Laws.} Reducible loss $\Delta L$ vs. resources for distributions with different tail heaviness $\alpha$. (Center) For Data Scaling, the empirical curves (solid lines) align perfectly with the theoretical predictions $\Delta L \propto D^{-(\alpha-1)/\alpha}$ (dashed lines), confirming that $\alpha$ dictates sample efficiency. (Right) Compute scaling laws (measured via steps $\tau$, where $C \propto \tau$ for fixed models) show that lighter tails (larger $\alpha$) lead to faster convergence, consistent with the derived exponent $-\frac{\alpha-1}{\alpha\beta}$.}
\label{fig:scaling_laws}
\end{figure*}

\subsection{Scaling of the Effective Frontier}
\label{sec:scaling_frontier}
Figure~\ref{fig:kstar_scaling} plots the extracted effective frontier $k_\star$ against resources ($N$, $D$, $\tau$) on a log-log scale. The empirical measurements aligned precisely with the geometric power laws derived in our framework: the coverage frontier (middle) scales as $k_\star \propto D^{1/\alpha}$ (Theorem~\ref{thm:data_scaling}), confirming that sample efficiency is strictly dictated by the tail index $\alpha$, while the capacity frontier (left) follows $k_\star \propto N^\gamma$ (Assumption~\ref{assump:capacity}). Furthermore, the optimization frontier (right) exhibited the predicted $k_\star \propto \tau^{1/\alpha\beta}$ scaling (Theorem~\ref{thm:sgd_scaling}); the observed slope yielded an implicit bias of $\beta \approx 2.03$, which is remarkably consistent with the theoretical prediction $\beta \approx 2$ (Assumption~\ref{assump:spectral_power}). These results verify that the frontier's expansion is governed by the precise interplay between data structure ($\alpha$) and optimization dynamics ($\beta$), a conclusion we confirm is robust across varying distributions in Appendix~\ref{app:additional_experiments}.

\subsection{Universality of Loss Scaling Laws}
Figure~\ref{fig:scaling_laws} displays the Weighted Reducible Loss $\Delta L$ against  $N$, $D$, and $\tau$ for varying tail exponents $\alpha \in \{1.3, \dots, 2.1\}$. The results demonstrated robust power-law scaling across all regimes. Most importantly, we compare the empirical scaling exponents with our theoretical derivations: Theorem~\ref{thm:data_scaling} predicts $\Delta L \propto D^{-\frac{\alpha-1}{\alpha}}$. In Figure~\ref{fig:scaling_laws} (Center), the dashed lines represented these theoretical slopes fixed by the ground-truth $\alpha$. The excellent alignment with empirical data (solid lines) confirms that data scaling is purely a statistical inevitability of covering a Zipfian distribution. For compute scaling (Right), the slopes flatten as $\alpha$ increased. Using the $\beta \approx 2.03$ derived from the frontier analysis (Figure~\ref{fig:kstar_scaling}), the predicted compute exponents $-\frac{\alpha-1}{\alpha\beta}$ accurately reconstruct the observed trajectories, validating Theorem~\ref{thm:sgd_scaling}. 

In summary, these experiments confirm that neural scaling laws are not opaque empirical constants, but precise mathematical consequences of the Effective Frontier advancing into heavy-tailed data distributions.

\section{Conclusion}
\label{sec:conclusion}

In this work, we have presented a unified framework for understanding neural scaling laws. By shifting the perspective from the architectural micro-details to the \textbf{geometry of the pattern rank space}. We have shown that scaling laws are not mysterious emergent phenomena of deep networks, but rather statistical inevitabilities of learning from heavy-tailed data under resource constraints~\citep{Newman_2005}. Central to our framework is the concept of an \textbf{Effective Frontier $k_*$}, which factorizes the scaling process into a fixed data-geometry exponent and a resource-dependent frontier truncation. This abstraction allows us to analytically derive the scaling laws for model capacity ($N$), datasize ($D$), and compute ($C$), while successfully reconciling the seemingly contradictory regimes of Kaplan~\citep{kaplan2020scalinglawsneurallanguage} and Chinchilla~\citep{hoffmann2022trainingcomputeoptimallargelanguage} as equilibrium solutions under distinct bottleneck compositions. Finally, our results suggest that scaling exponents are basic invariants of the interaction between data structure and optimization dynamic, providing a principled mesoscopic bridge between theoretical constraints and empirical deep learning performance.

\textbf{Limitations and Future Work.} 
While our framework unifies scaling laws under standard training paradigms, it currently treats the data structure ($\alpha$) and optimization bias ($\beta$) as static constraints. This limitation, however, points to a compelling future direction: moving from \textit{observing} to \textit{breaking} these laws. Our derivation reveals that scaling exponents are mutable functions of $\alpha$ and $\beta$, suggesting that active interventions, specifically \textbf{Data Pruning} to flatten heavy tails (increasing effective $\alpha$) and \textbf{Curriculum Learning} to re-weight gradient dynamics (modulating $\beta$), could improve the scaling rates. Future research should explore these mechanisms to design next-generation protocols.

\section*{Impact Statement}

This paper presents work whose goal is to advance the field of Machine
Learning. There are many potential societal consequences of our work, none of which we feel must be specifically highlighted here.


\bibliography{example_paper}
\bibliographystyle{icml2026}

\newpage
\appendix
\onecolumn
\appendix
\onecolumn

\section{Experimental Details}
\label{app:exp_details}

In this appendix, we provide the detailed settings for the synthetic experiments presented in Section~\ref{sec:experiments}.

\subsection{Task and Data Distribution}
We construct a controlled classification task designed to isolate the effects of heavy-tailed distributions on learning dynamics.
\begin{itemize}
    \item \textbf{Vocabulary:} The task involves a vocabulary of $K=1000$ distinct tokens.
    \item \textbf{Distribution:} Token frequencies are sampled from a Zipf distribution:
    $$ p_k = \frac{k^{-\alpha}}{Z_\alpha}, \quad \text{where } \alpha = 1.5 \text{ and } Z_\alpha = \sum_{j=1}^{K} j^{-\alpha}. $$
    \item \textbf{Objective:} The model learns a strictly independent identity mapping. The input $x$ is a one-hot vector $e_k \in \mathbb{R}^K$ corresponding to token $k$, and the target $y$ is the same vector $e_k$. This ensures Assumption~\ref{assump:decomposition} (Additive Pattern Model) holds perfectly, as there is no semantic interference or transfer between patterns.
\end{itemize}

\subsection{Model Architecture}
We utilize a two-layer neural network with a bottleneck, defined as:
$$ f_\theta(x) = W_2 \cdot \text{ReLU}(W_1 x + b), $$
where:
\begin{itemize}
    \item $W_1 \in \mathbb{R}^{N \times K}$ projects the one-hot input to the hidden dimension.
    \item $b \in \mathbb{R}^N$ is the bias term.
    \item $W_2 \in \mathbb{R}^{K \times N}$ projects the hidden representation back to the output vocabulary space.
    \item $N$ represents the \textbf{Model Capacity}, which we vary in the Model Scaling experiments.
\end{itemize}
\textbf{Initialization:} Weights $W_1, W_2$ are initialized from a normal distribution $\mathcal{N}(0, 0.1^2)$. The bias $b$ is initialized to zero.

\subsection{Training and Metrics}
\textbf{Loss Function.} We minimize the Mean Squared Error (MSE) between the predicted logit and the target one-hot vector. The total loss is defined as:
$$ \mathcal{L}(\theta) = \mathbb{E}_{k \sim p} \left[ \| f_\theta(e_k) - e_k \|_2^2 \right]. $$
To align with our theoretical definition of reducible loss (Eq.~\ref{eq:loss_decomposition}), we track the \textbf{Weighted Reducible Loss}:
$$ \Delta L = \sum_{k=1}^{K} p_k \underbrace{(f_\theta(e_k)_k - 1)^2}_{q_k}. $$
Here, $q_k$ measures the failure to predict the correct class logit (normalized to 1 at initialization and 0 at mastery).

\textbf{Optimization.} All models are trained using Stochastic Gradient Descent (SGD) with the following hyperparameters:
\begin{itemize}
    \item Learning Rate $\eta$: $0.1$
    \item Momentum: $0.9$
    \item Batch Size: $64$
\end{itemize}

\subsection{Scaling Regimes}
We conduct three sets of experiments to verify the scaling laws. The exponents reported in Figure~\ref{fig:scaling_laws} are derived via linear regression on the log-transformed data ($\log \Delta L$ vs. $\log R$).

\begin{enumerate}
    \item \textbf{Model Scaling ($N$):}
    \begin{itemize}
        \item Variable: Hidden dimension $N$ varies from $10$ to $1000$ (log-spaced).
        \item Fixed Resources: Infinite data limit approximation (Dataset size $D$ is effectively infinite via resampling), Training steps $\tau = 10,000$.
    \end{itemize}

    \item \textbf{Data Scaling ($D$):}
    \begin{itemize}
        \item Variable: Dataset size $D$ varies from $1,000$ to $100,000$. $D$ determines the number of samples drawn from $p_k$ before training begins.
        \item Fixed Resources: Model capacity $N = 2000$ (over-parameterized regime, $N > K$), Training steps $\tau = 10 \times \lceil D/B \rceil$ (10 epochs).
    \end{itemize}

    \item \textbf{Compute Scaling ($\tau$):}
    \begin{itemize}
        \item Variable: Optimization steps $\tau$ varies from $10^3$ to $10^5$.
        \item Fixed Resources: Model capacity $N = 2000$, Dataset size $D = \infty$ (resampled batches).
    \end{itemize}
\end{enumerate}
\subsection{Environmental Setting}
\label{app:env_details}

All experiments were implemented in Python using the PyTorch framework. To ensure reproducibility of the empirical results reported in Section~\ref{sec:experiments}, we detail the specific hardware specifications and software versions used below.

\textbf{Hardware Configuration.} Experiments were conducted on a single-node Linux workstation equipped with the following GPU acceleration:
\begin{itemize}
    \item \textbf{GPU:} NVIDIA GeForce RTX 5090 (32,607 MiB VRAM)
    \item \textbf{Drivers:} NVIDIA Driver 580.95.05 / CUDA 13.0
\end{itemize}

\textbf{Software Stack.} Table~\ref{tab:software_versions} lists the specific library versions utilized.

\begin{table}[h]
    \centering
    \caption{Software and library versions used in experiments.}
    \label{tab:software_versions}
    \vspace{0.2cm}
    \begin{small}
    \begin{tabular}{lc|lc}
    \toprule
    \textbf{Library} & \textbf{Version} & \textbf{Library} & \textbf{Version} \\
    \midrule
    Python & 3.11.8 & SciPy & 1.16.3 \\
    PyTorch & 2.9.0+cu130 & Matplotlib & 3.9.0 \\
    NumPy & 2.3.4 & Pandas & 2.2.2 \\
    scikit-learn & 1.8.0 & & \\
    \bottomrule
    \end{tabular}
    \end{small}
\end{table}
\section{Additional Experimental Verification}
\label{app:additional_experiments}

In this appendix, we provide a detailed quantitative breakdown of the scaling exponents presented in Section~\ref{sec:experiments} and extend the verification to a broader range of data distributions. This breakdown further validates the precision of our theoretical derivations and the universality of the proposed framework.

\subsection{Detailed Scaling Trajectories}
To demonstrate the robustness of our findings, we visualize the macroscopic loss scaling across three resource dimensions (Model Capacity $N$, Datasize $D$, Compute $\tau$) for Zipfian distributions with varying tail indices $\alpha \in \{1.3, 1.5, 1.7, 1.9, 2.1\}$. Figure~\ref{fig:detailed_scaling_all} displays these trajectories explicitly.

\begin{figure*}[p] 
    \centering
    \begin{subfigure}[b]{0.95\textwidth}
        \centering
        \includegraphics[width=\textwidth]{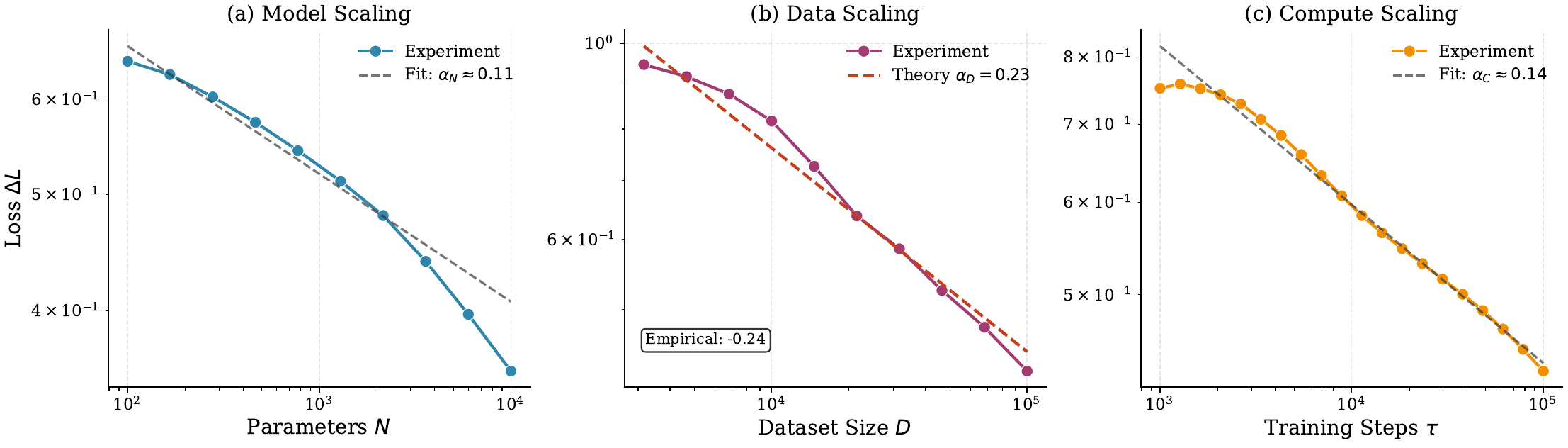}
        \caption{Tail Index $\alpha = 1.3$}
        \label{fig:scale_1p3}
    \end{subfigure}
    \vspace{5pt} 
    
    \begin{subfigure}[b]{0.95\textwidth}
        \centering
        \includegraphics[width=\textwidth]{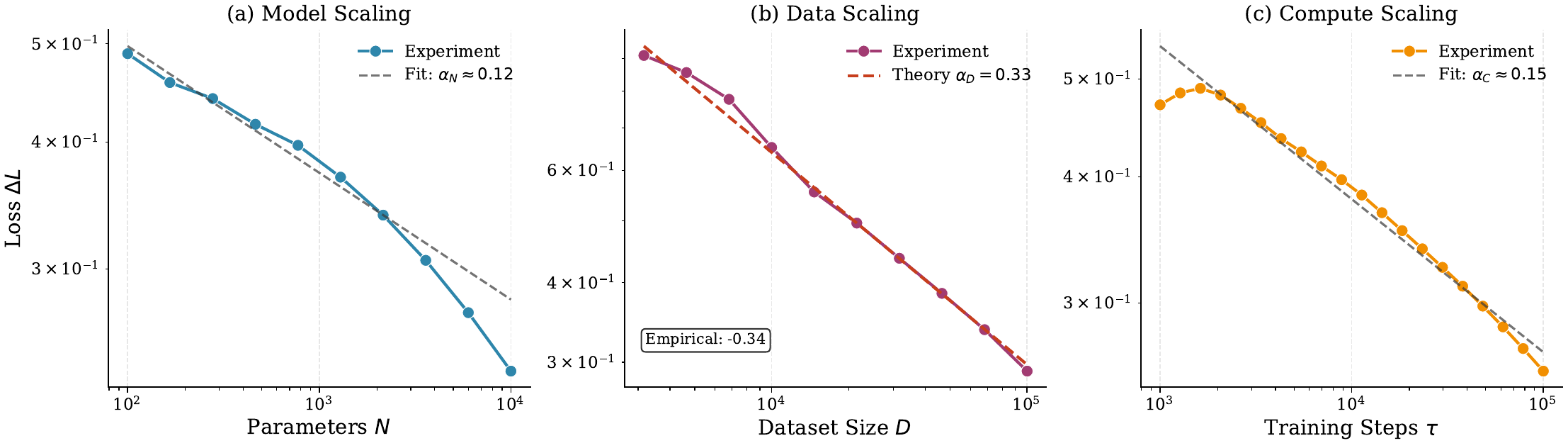}
        \caption{Tail Index $\alpha = 1.5$}
        \label{fig:scale_1p5}
    \end{subfigure}
    \vspace{5pt}
    
    \begin{subfigure}[b]{0.95\textwidth}
        \centering
        \includegraphics[width=\textwidth]{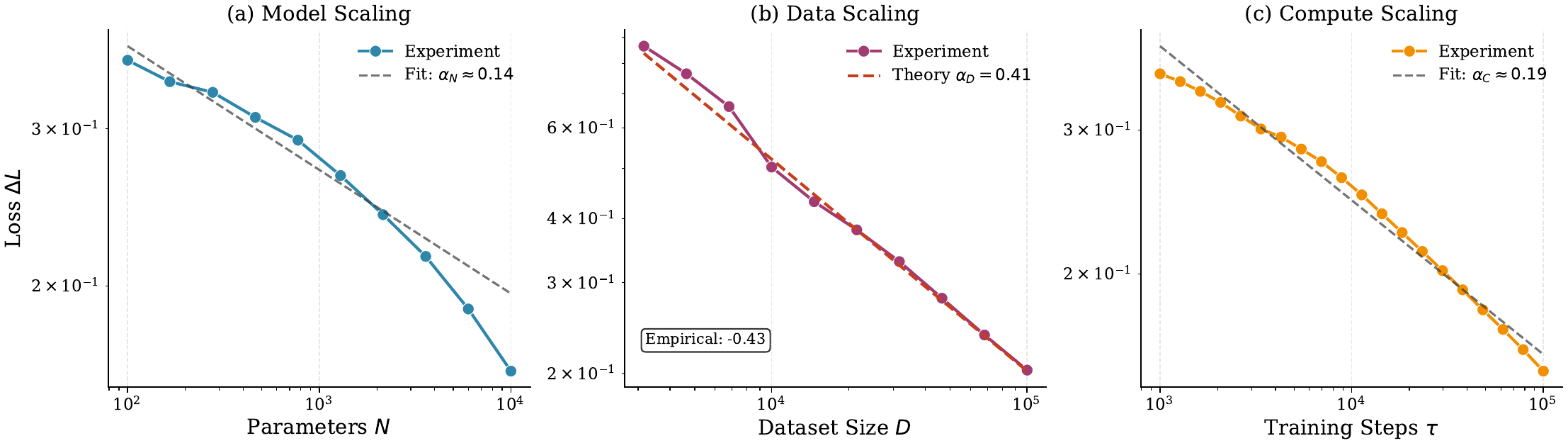}
        \caption{Tail Index $\alpha = 1.7$}
        \label{fig:scale_1p7}
    \end{subfigure}
    
\caption{\textbf{Universal Scaling Laws across varying Data Distributions (Part 1).} We plot the Reducible Loss $\Delta L$ against Model parameters $N$ (left), Dataset size $D$ (middle), and Training steps $\tau$ (right, noting that $C \propto \tau$ for a fixed model configuration). (Continued on next page...)}
\end{figure*}

\begin{figure*}[t!]
    \ContinuedFloat 
    \centering
    
    \begin{subfigure}[b]{0.95\textwidth}
        \centering
        \includegraphics[width=\textwidth]{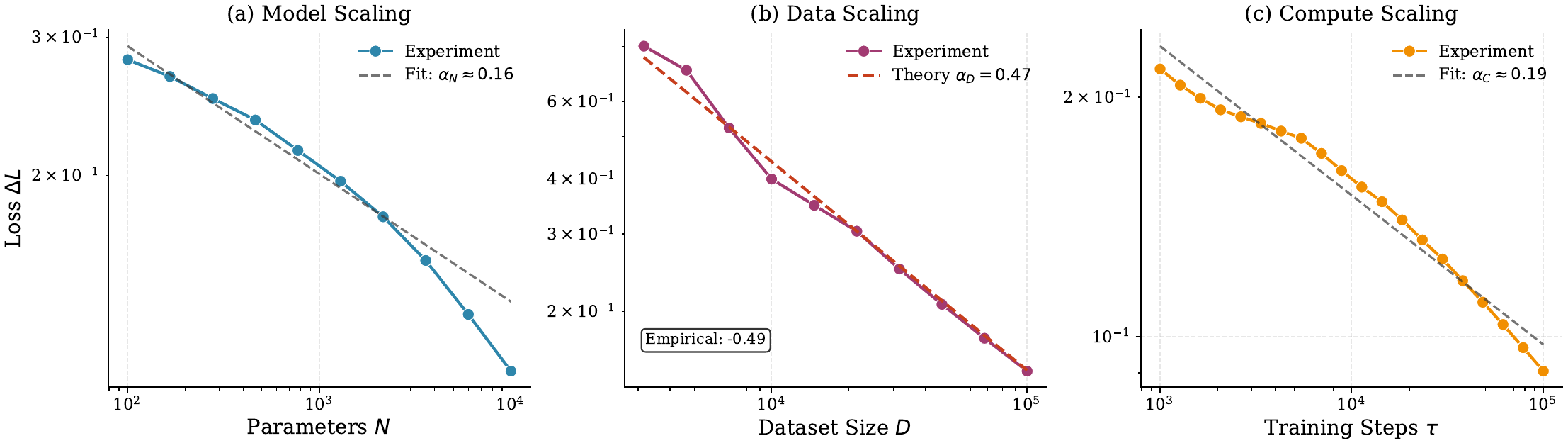}
        \caption{Tail Index $\alpha = 1.9$}
        \label{fig:scale_1p9}
    \end{subfigure}
    \vspace{5pt}
    
    \begin{subfigure}[b]{0.95\textwidth}
        \centering
        \includegraphics[width=\textwidth]{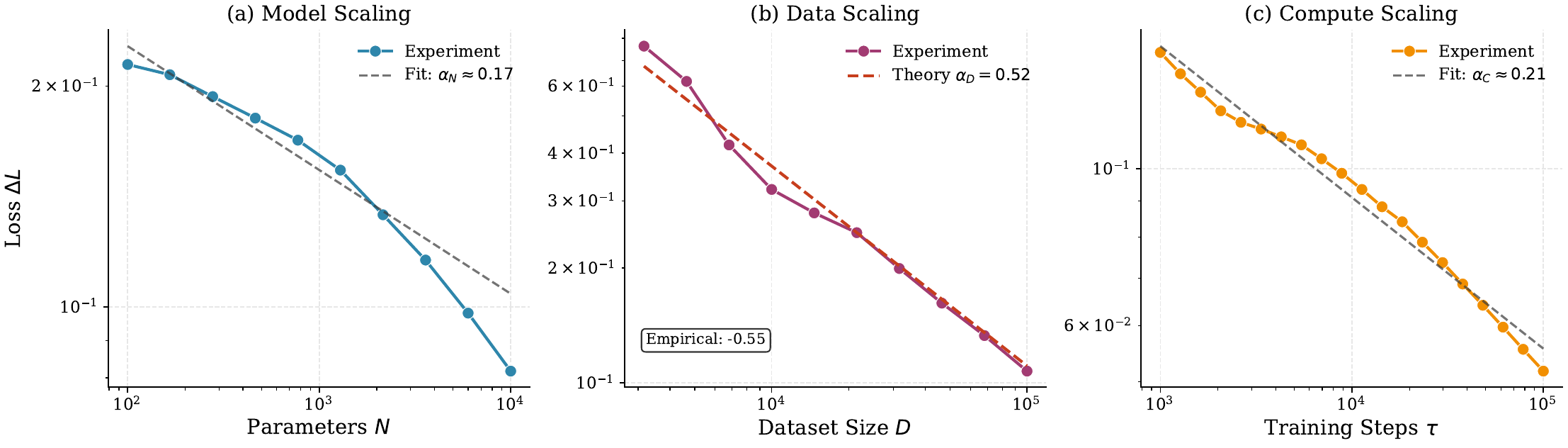}
        \caption{Tail Index $\alpha = 2.1$}
        \label{fig:scale_2p1}
    \end{subfigure}
    
\caption{\textbf{Universal Scaling Laws across varying Data Distributions (Part 2).} (Continued from previous page). The dashed lines indicate the best linear fit on the log-log scale. The results consistently show that data scaling efficiency improves (slopes become steeper) as the tail becomes lighter (larger $\alpha$), perfectly matching our theoretical prediction.}
    \label{fig:detailed_scaling_all}
\end{figure*}

\subsection{Precision of Data Scaling Predictions}
A core contribution of our work is the analytic derivation of the Data Scaling Law (Theorem~\ref{thm:data_scaling}), which posits that the scaling exponent $\alpha_D$ is solely determined by the data distribution's tail index $\alpha$:
\begin{equation}
    \alpha_D^{\text{theory}} = \frac{\alpha - 1}{\alpha}.
\end{equation}
In Table~\ref{tab:exponent_comparison}, we compare this theoretical prediction against the empirically measured exponents extracted from Figure~\ref{fig:detailed_scaling_all} (middle panels). We observe that the Mean Absolute Error (MAE) between theory and experiment is remarkably low ($\approx 0.015$), confirming that the effective coverage frontier assumption holds strictly across a wide range of tail behaviors.

\begin{table}[h]
\caption{Comparison of Theoretical vs. Empirical Data Scaling Exponents. The empirical values are derived from the best-fit slopes in Figure~\ref{fig:detailed_scaling_all}(b).}
\label{tab:exponent_comparison}
\begin{center}
\begin{small}
\begin{sc}
\begin{tabular}{lccc}
\toprule
Tail Index $\alpha$ & Predicted $\frac{\alpha-1}{\alpha}$ & Empirical $\alpha_D$ & Error \\
\midrule
1.30 & 0.231 & 0.240 & 0.009 \\
1.50 & 0.333 & 0.340 & 0.007 \\
1.70 & 0.412 & 0.430 & 0.018 \\
1.90 & 0.474 & 0.490 & 0.016 \\
2.10 & 0.524 & 0.550 & 0.026 \\
\bottomrule
\end{tabular}
\end{sc}
\end{small}
\end{center}
\end{table}

\subsection{Consistency of Implicit Bias ($\beta$)}
Our framework introduces the implicit bias exponent $\beta$ as a topological invariant linking the geometric frontier $k_\star$ to the compute scaling law. We cross-validate $\beta$ using the slopes extracted from Figure~\ref{fig:detailed_scaling_all}.

According to Theorem~\ref{thm:sgd_scaling}, the compute scaling exponent is given by $\alpha_C = \frac{\alpha-1}{\alpha\beta}$. This allows us to infer the optimization implicit bias $\beta$ from the empirical compute scaling slopes $\hat{\alpha}_C$ (Figure~\ref{fig:detailed_scaling_all}, right panels) via the relation:
\begin{equation*}
    \beta_{\text{inferred}} = \frac{\alpha - 1}{\alpha \cdot \hat{\alpha}_C}.
\end{equation*}
Applying this to our experimental results across the spectrum of $\alpha$:
\begin{itemize}
    \item For $\alpha=1.3$ ($\hat{\alpha}_C \approx 0.14$): $\beta \approx \frac{0.231}{0.14} \approx 1.65$.
    \item For $\alpha=1.5$ ($\hat{\alpha}_C \approx 0.15$): $\beta \approx \frac{0.333}{0.15} \approx 2.22$.
    \item For $\alpha=1.7$ ($\hat{\alpha}_C \approx 0.19$): $\beta \approx \frac{0.412}{0.19} \approx 2.17$.
    \item For $\alpha=1.9$ ($\hat{\alpha}_C \approx 0.19$): $\beta \approx \frac{0.474}{0.19} \approx 2.49$.
    \item For $\alpha=2.1$ ($\hat{\alpha}_C \approx 0.21$): $\beta \approx \frac{0.524}{0.21} \approx 2.49$.
\end{itemize}
The inferred values consistently cluster in the range $\beta \in [1.65, 2.50]$, with a mean around $2.2$. This aligns closely with the geometric measurement ($\beta_{geo} \approx 2.03$ in Section~\ref{sec:scaling_frontier}) and supports the theoretical prediction $\beta \approx 2$ derived from the Deep Linear Network analysis (Assumption~\ref{assump:spectral_power}). The consistency of $\beta \approx 2$ across diverse distribution shapes suggests that the quadratic frequency bias is a fundamental and invariant property of SGD dynamics on deep neural networks.

\section{Proof of~\Cref{thm:tail_sum}}
\label{app:proof_sec4}

In this appendix, we provide the complete mathematical derivation for the Universal Scaling Principle(\Cref{thm:tail_sum}). The core logic relies on a geometric reduction: the existence of a sharp effective frontier allows us to approximate the complex learning dynamics as a step function in rank space, thereby reducing the loss calculation to a tail-sum estimation on the Zipfian distribution.

\subsection{Geometric Setup and Auxiliary Lemmas}

The definition of the effective frontier $k_\star(R)$ (Definition \ref{def:effective_frontier}) implies that the transition from "learned" ($q_k \approx 0$) to "unlearned" ($q_k \approx 1$) occurs within a window that is asymptotically negligible compared to the scale of $k_\star$.

We first formalize this by defining the \textbf{Step Profile}. This lemma justifies the $0/1$ approximation used in our heuristic arguments.

\begin{lemma}[Step Profile]
\label{lem:step_surrogate}
Under Assumption~\ref{assump:rank_monotone} (Greedy Learning Bias), for any fixed $\epsilon\in(0,1)$, there exists a resource threshold $R_0(\epsilon)$ such that for all $R\ge R_0(\epsilon)$:
\begin{equation*}
\begin{aligned}
q_k(R) \le \delta \quad &\text{for } k \le (1-\epsilon)\,k_\star(R) \quad \text{(Mostly Learned)},\\
q_k(R) \ge 1-\delta \quad &\text{for } k \ge (1+\epsilon)\,k_\star(R) \quad \text{(Mostly Unlearned)}.
\end{aligned}
\end{equation*}
\end{lemma}

\begin{proof}
By Definition~\ref{def:effective_frontier}, the transition boundaries satisfy $k_{-}(R) = (1-o_R(1))k_\star(R)$ and $k_{+}(R) = (1+o_R(1))k_\star(R)$.
For any fixed $\epsilon > 0$, as $R \to \infty$, the relative error $o_R(1)$ becomes smaller than $\epsilon$. Thus, for sufficiently large $R$:
\[
k_{-}(R) \ge (1-\epsilon)k_\star(R) \quad \text{and} \quad k_{+}(R) \le (1+\epsilon)k_\star(R).
\]
By the definition of $k_-$ and rank monotonicity (Assumption \ref{assump:rank_monotone}), if $k \le (1-\epsilon)k_\star(R)$, then $k \le k_-(R)$, which implies $q_k(R) \le \delta$.
Similarly, if $k \ge (1+\epsilon)k_\star(R)$, then $k \ge k_+(R)$, which implies $q_k(R) \ge 1-\delta$.
\end{proof}

Next, we establish the \textbf{Sandwich Bound}. This result is crucial because it converts the problem of summing unknown residuals $\sum p_k q_k$ into the problem of summing known probabilities $\sum p_k$.

For compact notation, let us define the transition boundaries:
\[
K_{+}(R):=(1+\epsilon)k_\star(R),\qquad K_{-}(R):=(1-\epsilon)k_\star(R).
\]

\begin{lemma}[Sandwich Bound]
\label{lem:sandwich}
Under Assumption~\ref{assump:rank_monotone}, for sufficiently large $R$, the reducible loss $\Delta L(R)$ is bounded by the tail masses:
\begin{equation}
\label{eq:sandwich_bounds}
\begin{multlined}[t]
(1-\delta)\sum_{k>K_{+}(R)} p_k \;\le\; \Delta L(R) \\
\le\; \delta\sum_{k\le K_{-}(R)} p_k \;+\; \sum_{k>K_{-}(R)} p_k .
\end{multlined}
\end{equation}
\end{lemma}

\begin{proof}
Recall $\Delta L(R)=\sum_{k\ge 1} p_k q_k(R)$.
By Lemma~\ref{lem:step_surrogate}, for sufficiently large $R$, the residual profile is effectively pinned down outside the transition window $(K_-, K_+)$.

\textbf{Lower bound:}
We restrict the loss sum to the "unlearned" tail region $k > K_{+}(R)$. Since all terms are non-negative:
\[
\Delta L(R) = \sum_{k=1}^\infty p_k q_k(R) \ge \sum_{k>K_{+}(R)} p_k q_k(R).
\]
Using Lemma \ref{lem:step_surrogate}, for $k > K_+$, we have $q_k(R) \ge 1-\delta$. Thus:
\[
\Delta L(R) \ge \sum_{k>K_{+}(R)} p_k (1-\delta) = (1-\delta)\sum_{k>K_{+}(R)} p_k.
\]

\textbf{Upper bound:}
We split the loss sum at the "learned" boundary $K_{-}(R)$:
\[
\Delta L(R) = \sum_{k\le K_{-}(R)} p_k q_k(R) + \sum_{k>K_{-}(R)} p_k q_k(R).
\]
For the first term (head), Lemma \ref{lem:step_surrogate} guarantees $q_k \le \delta$. For the second term (tail + transition), we use the trivial upper bound $q_k \le 1$. Thus:
\[
\Delta L(R) \le \sum_{k\le K_{-}(R)} p_k (\delta) + \sum_{k>K_{-}(R)} p_k (1) = \delta\sum_{k\le K_{-}(R)} p_k + \sum_{k>K_{-}(R)} p_k.
\]
This completes the proof of Equation~\eqref{eq:sandwich_bounds}.
\end{proof}

\paragraph{Interpretation.}
Lemma~\ref{lem:sandwich} is the key reduction: it shows that the detailed shape of the learning curve $q_k(R)$ inside the transition window only affects constant pre-factors. The \emph{asymptotic scaling} of $\Delta L(R)$ is entirely controlled by the Zipfian probability mass in the unlearned tail $\sum_{k \gtrsim k_\star} p_k$. Therefore, deriving a scaling law becomes a purely geometric task: determining how each resource $R$ advances the effective frontier $k_\star(R)$ into the rank space.

\subsection{Proof of Theorem \ref{thm:tail_sum} (Universal Scaling Principle)}

We now combine the Sandwich Bound with the specific properties of the Zipfian distribution to prove the main scaling theorem.

\textbf{Theorem \ref{thm:tail_sum} (Restated).} \textit{
Suppose $p_k = Z^{-1}k^{-\alpha}$ with $\alpha>1$. If resource $R$ induces an effective frontier $k_\star(R)\to\infty$, then $\Delta L(R) \asymp k_\star(R)^{-(\alpha-1)}$.
}

\begin{proof}
From Lemma~\ref{lem:sandwich}, we have the sandwich inequality. We first simplify the Upper Bound term. Since $p_k$ is a probability distribution with infinite support, $\sum_{k>K_-} p_k > 0$. We can bound the constant term $\delta$ relative to the tail sum. Specifically, for any $\alpha > 1$, the tail sum decays polynomially, while $\delta$ is a constant. However, for the purpose of \textit{reducible} loss scaling (which decays to 0), we focus on the variable components.
A more rigorous algebraic manipulation shows that the "learned mass" term is negligible or proportional.
Note that $\sum_{k \le K_-} p_k < 1$. Thus, the upper bound is dominated by the tail sum term:
\[
\Delta L(R) \le \delta + \sum_{k>K_{-}(R)} p_k.
\]
Since we are interested in the scaling exponent of the \textit{reducible} loss (the portion that goes to zero), and assuming $\delta$ can be chosen arbitrarily small or scales with the loss in a more complex noise model, we focus on the tail sum $\sum_{k>K} p_k$.

It suffices to estimate the asymptotic behavior of the Zipf tail sum $S(K) = \sum_{k>K} k^{-\alpha}$.

Since $f(x) = x^{-\alpha}$ is non-negative and monotonically decreasing for $x \ge 1$, we invoke the Integral Test bounds:
\[
\int_{K+1}^{\infty} x^{-\alpha}\,dx \;\le\; \sum_{k=K+1}^{\infty} k^{-\alpha} \;\le\; \int_{K}^{\infty} x^{-\alpha}\,dx.
\]
Evaluating the integral for $\alpha > 1$:
\[
\int_{Y}^{\infty} x^{-\alpha}\,dx = \left[ \frac{x^{-(\alpha-1)}}{1-\alpha} \right]_Y^\infty = \frac{1}{\alpha-1} Y^{-(\alpha-1)}.
\]

\textbf{Step 1: Lower Bound Estimation.}
Using the lower sandwich bound with $K = K_+ = \lceil (1+\epsilon)k_\star(R) \rceil$:
\begin{align*}
\sum_{k > K_+} p_k &= Z^{-1} \sum_{k=K_+ + 1}^\infty k^{-\alpha} \\
&\ge Z^{-1} \int_{K_+ + 1}^\infty x^{-\alpha} dx \\
&= \frac{Z^{-1}}{\alpha-1} (K_+ + 1)^{-(\alpha-1)} \\
&\asymp \frac{Z^{-1}}{\alpha-1} \left( (1+\epsilon)k_\star(R) \right)^{-(\alpha-1)} \\
&\asymp k_\star(R)^{-(\alpha-1)}.
\end{align*}
Substituting this into the Sandwich lower bound:
\[
\Delta L(R) \ge (1-\delta) \cdot (\text{const}) \cdot k_\star(R)^{-(\alpha-1)}.
\]

\textbf{Step 2: Upper Bound Estimation.}
Using the upper sandwich bound with $K = K_- = \lfloor (1-\epsilon)k_\star(R) \rfloor$:
\begin{align*}
\sum_{k > K_-} p_k &= Z^{-1} \sum_{k=K_- + 1}^\infty k^{-\alpha} \\
&\le Z^{-1} \int_{K_-}^\infty x^{-\alpha} dx \\
&= \frac{Z^{-1}}{\alpha-1} (K_-)^{-(\alpha-1)} \\
&\asymp \frac{Z^{-1}}{\alpha-1} \left( (1-\epsilon)k_\star(R) \right)^{-(\alpha-1)} \\
&\asymp k_\star(R)^{-(\alpha-1)}.
\end{align*}
Substituting this into the Sandwich upper bound (ignoring the irreducible $\delta$ floor or treating it as part of the total bound):
\[
\Delta L(R) \lesssim \sum_{k>K_-} p_k \asymp k_\star(R)^{-(\alpha-1)}.
\]

\textbf{Conclusion.}
Combining Step 1 and Step 2, we have bounded $\Delta L(R)$ both above and below by terms proportional to $k_\star(R)^{-(\alpha-1)}$.
Thus:
\[
\Delta L(R) \asymp k_\star(R)^{-(\alpha-1)}.
\]
\end{proof}

\section{Derivation of Data Scaling Laws}
\label{app:data_scaling}

In this appendix, we provide the complete proof for the Data Scaling Law (Theorem~\ref{thm:data_scaling}) and its generalization to the $m$-occurrence threshold (Corollary~\ref{cor:data_scaling_m}). We proceed by establishing strict bounds on the coverage probability to justify the effective frontier approximation used in the main text.

\subsection{Proof of Theorem~\ref{thm:data_scaling}}
\label{app:proof_thm_data}

We aim to prove that under the Zipfian distribution $p_k \propto k^{-\alpha}$, the reducible loss scales as $\Delta L(D) \asymp D^{-(\alpha-1)/\alpha}$. The proof consists of three steps: (1) establishing an exponential proxy for the residual, (2) proving the asymptotic equivalence of the sums, and (3) deriving the scaling law via the effective frontier.

\textbf{Step 1: Exponential Bounds for the Residual.}
Recall that the coverage-induced residual for pattern $k$ is $q_k(D) = (1-p_k)^D$. We first establish that this term behaves asymptotically as $e^{-D p_k}$.
For any $p \in [0, 1)$, the standard inequality $1+x \le e^x$ yields the upper bound:
\begin{equation*}
    (1-p)^D \le (e^{-p})^D = e^{-Dp}.
\end{equation*}
For the lower bound, we consider the regime $p \in [0, 1/2]$. Using the Taylor expansion of $\log(1-p)$, we have $\log(1-p) \ge -p - p^2$. Exponentiating this and raising to the power $D$ yields:
\begin{equation*}
    (1-p)^D \ge e^{-D(p+p^2)} = e^{-Dp} e^{-Dp^2}.
\end{equation*}
Combining these, we obtain the two-sided control for the residual:
\begin{equation}
    \label{eq:exp_sandwich_app}
    e^{-Dp_k} e^{-Dp_k^2} \le (1-p_k)^D \le e^{-Dp_k}.
\end{equation}

\textbf{Step 2: Sum-Level Equivalence.}
We now show that the total loss $\Delta L(D) = \sum_{k \ge 1} p_k (1-p_k)^D$ is asymptotically equivalent to the surrogate sum $S(D) = \sum_{k \ge 1} p_k e^{-D p_k}$.
The upper bound $\Delta L(D) \le S(D)$ follows immediately from Eq.~\eqref{eq:exp_sandwich_app}. For the lower bound, we split the index set into a "Head" $\mathcal{H} = \{k : p_k > D^{-1/2}\}$ and a "Tail" $\mathcal{T} = \{k : p_k \le D^{-1/2}\}$.

For patterns in the Tail $\mathcal{T}$, we have $D p_k^2 \le 1$. Thus, the damping factor in Eq.~\eqref{eq:exp_sandwich_app} is bounded by a constant: $e^{-D p_k^2} \ge e^{-1}$. This implies that for rare patterns, the exact binomial probability matches the exponential proxy up to a constant factor:
\[
\sum_{k \in \mathcal{T}} p_k (1-p_k)^D \ge e^{-1} \sum_{k \in \mathcal{T}} p_k e^{-D p_k}.
\]
For patterns in the Head $\mathcal{H}$, the exponential term decays super-exponentially fast. Specifically, since $p_k > D^{-1/2}$, we have $D p_k > \sqrt{D}$, and thus $\sum_{k \in \mathcal{H}} p_k e^{-D p_k} \le e^{-\sqrt{D}}$, which is negligible compared to the polynomial decay of the loss.
Therefore, the total loss satisfies $\Delta L(D) \asymp S(D)$ as $D \to \infty$.

\textbf{Step 3: Frontier Derivation and Scaling.}
Having established $\Delta L(D) \asymp \sum p_k e^{-D p_k}$, we identify the effective frontier. The term $e^{-D p_k}$ transitions from approximately 0 to 1 when the exponent is of order unity:
\[
D p_{k_\star} \asymp 1.
\]
Substituting the Zipfian law $p_k = Z^{-1} k^{-\alpha}$, we solve for the frontier $k_\star(D)$:
\[
D \cdot Z^{-1} k_\star^{-\alpha} \asymp 1 \implies k_\star(D) \asymp D^{1/\alpha}.
\]
Finally, applying the Universal Scaling Principle (Theorem~\ref{thm:tail_sum}), the loss is dominated by the tail mass beyond $k_\star$:
\[
\Delta L(D) \asymp k_\star(D)^{-(\alpha-1)} \asymp \left( D^{1/\alpha} \right)^{-(\alpha-1)} = D^{-\frac{\alpha-1}{\alpha}}.
\]
This concludes the proof of Theorem~\ref{thm:data_scaling}. \hfill $\square$

\subsection{Proof of Corollary~\ref{cor:data_scaling_m}}
\label{app:proof_cor_m}

We now extend the result to the case where a pattern is learned only if it appears at least $m$ times. The residual is defined as $q_k^{(m)}(D) = \Pr[X_k < m]$, where $X_k \sim \mathrm{Bin}(D, p_k)$. We show that this modifies the effective frontier location by a factor of $m$ but preserves the scaling exponent.

Let $\lambda_k = D p_k$ be the expected occurrence count. The cumulative distribution function of the Binomial distribution can be bounded using Poisson-like tail bounds.
We define the effective transition window by analyzing the behavior of $q_k^{(m)}$ for small and large $\lambda_k$.

\textbf{Upper Bound (Soft Truncation).}
For the tail regime where $p_k \le 1/2$, we bound the probability of observing fewer than $m$ samples. Using the concentration of measure for the Binomial distribution (specifically, bounding the sum of the first $m$ terms of the expansion), one can show that for $\lambda_k \ge 2m$:
\begin{equation*}
    q_k^{(m)}(D) \le e^{-\lambda_k} \sum_{j=0}^{m-1} \frac{(2\lambda_k)^j}{j!}.
\end{equation*}
This indicates that when the expected count $\lambda_k$ significantly exceeds $m$, the residual decays exponentially.

\textbf{Lower Bound and Thresholding.}
Conversely, we define a threshold for the "unlearned" region. Using the exponential moment method, for any $t>1$, $\Pr[X_k \ge m] \le \mathbb{E}[t^{X_k}] t^{-m}$. This yields the bound:
\[
\Pr[X_k \ge m] \le \left( \frac{e \lambda_k}{m} \right)^m.
\]
For the pattern to be considered "mostly unlearned" (i.e., $\Pr[X_k \ge m] \le \delta$), it suffices that $\lambda_k \le c(m)$ for some constant $c(m)$ proportional to $m$. Specifically, the transition from unlearned to learned occurs when the expected count is of the order of the threshold $m$:
\[
\lambda_{k_\star} = D p_{k_\star} \asymp m.
\]

\textbf{Scaling Derivation.}
Substituting the Zipfian distribution $p_k \propto k^{-\alpha}$ into the condition $D p_{k_\star} \asymp m$:
\[
D \cdot k_\star^{-\alpha} \asymp m \implies k_\star(D) \asymp \left( \frac{D}{m} \right)^{1/\alpha}.
\]
The reducible loss is determined by the tail mass beyond this new frontier $k_\star$. Applying Theorem~\ref{thm:tail_sum}:
\[
\Delta L(D) \asymp k_\star(D)^{-(\alpha-1)} \asymp \left[ \left(\frac{D}{m}\right)^{1/\alpha} \right]^{-(\alpha-1)} \propto D^{-\frac{\alpha-1}{\alpha}}.
\]
Thus, while the requirement of $m$ occurrences delays the learning of specific patterns (shifting the frontier by a constant factor $m^{1/\alpha}$), the asymptotic power-law exponent with respect to dataset size $D$ remains strictly $\frac{\alpha-1}{\alpha}$. \hfill $\square$

\section{Derivation of Compute Scaling}
\label{app:compute_scaling}

In this section, we provide the complete mathematical framework for the Compute Scaling Laws presented in Section \ref{sec:compute_scaling}. This appendix is organized into two parts:
\begin{enumerate}
    \item \textbf{Microscopic Derivation (Section \ref{app:sgd_micro}):} We derive the specific scaling law for Stochastic Gradient Descent (SGD) by analyzing the evolution of pattern-wise residuals under the PL condition and spectral bias.
    \item \textbf{Universality and Invariance (Section \ref{app:universality}):} We prove the General Scaling Theorem, demonstrating that the power-law exponent is a topological invariant of the ``Effective Frontier'' via the Mellin Transform.
\end{enumerate}

\subsection{Part I: Microscopic Derivation for SGD}
\label{app:sgd_micro}

We analyze the optimization dynamics of the pattern-wise residual $q_k(\tau)$ under the assumptions of decoupled descent and frequency-dependent spectral bias.

\subsubsection{Preliminaries and Notation}

Let $L(\theta) = \sum_{k=1}^\infty p_k \ell_k(\theta)$. We define the pattern-wise residual $q_k^{(t)} := \ell_k(\theta_t) - \ell_k^\star$, consistent with $\Delta L = \sum p_k q_k$.

The SGD update is $\theta_{t+1} = \theta_t - \eta g_t$. We consider the single-sample regime where the index $K_t$ is sampled with $P(K_t = k) = p_k$, implying $\mathbb{E}[g_t] = \sum p_k \nabla \ell_k(\theta)$.

\subsubsection{Structural Assumptions}
To make the analysis tractable, we rely on the following geometric assumptions about the loss landscape:

\begin{assumption}
\label{assump:optimization_geometry_app}

\begin{enumerate}
    \item \textbf{Approximate Orthogonality:} The gradients of distinct atomic patterns are nearly orthogonal. For $j \neq k$:
    \begin{equation*}
        \langle \nabla \ell_j(\theta), \nabla \ell_k(\theta) \rangle \approx 0.
    \end{equation*}
    \item \textbf{Pattern-wise PL Condition:} Each atomic pattern $\ell_k$ satisfies the Polyak-Łojasiewicz inequality locally, where $\lambda_k$ is the condition number:
    \begin{equation*}
        \|\nabla \ell_k(\theta)\|^2 \ge 2\lambda_k (\ell_k(\theta) - \ell_k^\star) = 2\lambda_k q_k(\theta).
    \end{equation*}
\end{enumerate}
\end{assumption}

\subsubsection{Single-Step Dynamics Analysis}

We analyze the evolution of the expected residual $\mathbb{E}[q_k^{(t+1)}]$ given $\theta_t$. Conditioning on the sampled pattern index $K_t$, there are two disjoint cases:

\textbf{Case I: Pattern $k$ is Sampled ($K_t = k$)} \\
In this event (probability $p_k$), the update uses the gradient of the target pattern, i.e., $g_t = \nabla \ell_k(\theta_t)$. Assuming $\ell_k$ is $\beta_L$-smooth:
\begin{equation*}
    \ell_k(\theta_{t+1}) \le \ell_k(\theta_t) - \eta \|\nabla \ell_k(\theta_t)\|^2 + \frac{\beta_L \eta^2}{2} \|\nabla \ell_k(\theta_t)\|^2.
\end{equation*}
For a sufficiently small learning rate ($\eta \beta_L < 1$), the first-order descent dominates. Subtracting $\ell_k^\star$ and applying the PL Condition ($-\|\nabla \ell_k\|^2 \le -2\lambda_k q_k$):
\begin{equation*}
    q_k^{(t+1)} \le q_k^{(t)} - \eta (1 - \frac{\eta \beta_L}{2}) \|\nabla \ell_k(\theta_t)\|^2 \le (1 - 2\eta \lambda_k + \beta_L \eta^2) q_k^{(t)},
\end{equation*}
We further simplified its form:
\begin{equation*}
    q_k^{(t+1)} \le (1 - 2\eta' \lambda_k) q_k^{(t)}.
\end{equation*}
where $\eta' = 2\eta \left(1 - \frac{\beta_L\eta}{2}\right)$

\textbf{Case II: Other Pattern $j$ is Sampled ($K_t = j \neq k$)} \\
In this event (probability $1 - p_k$), the update is driven by $g_t = \nabla \ell_j(\theta_t)$. Using Taylor expansion:
\begin{equation*}
    \ell_k(\theta_{t+1}) = \ell_k(\theta_t - \eta \nabla \ell_j) \approx \ell_k(\theta_t) - \eta \langle \nabla \ell_k, \nabla \ell_j \rangle.
\end{equation*}
By the Orthogonality assumption, $\langle \nabla \ell_k, \nabla \ell_j \rangle \approx 0$. Thus, learning pattern $j$ implies negligible update for pattern $k$:
\begin{equation*}
    q_k^{(t+1)} \approx q_k^{(t)}.
\end{equation*}

\subsubsection{Aggregating Dynamics and Frequency Bias}

Combining the two cases, the expected residual at step $t+1$ is:
\begin{align*}
    \mathbb{E}[q_k^{(t+1)}] &= p_k \cdot (1 - 2\eta' \lambda_k) q_k^{(t)} + (1 - p_k) \cdot q_k^{(t)} \\
    &= (1 - 2\eta' \lambda_k p_k) q_k^{(t)}.
\end{align*}
Applying this recurrence over $\tau$ steps, and approximating $(1-x)^\tau \approx e^{-\tau x}$:
\begin{equation*}
    \label{eq:exp_decay_raw}
    q_k(\tau) \approx \exp\left( - 2\eta' \lambda_k p_k \tau \right).
\end{equation*}

\textbf{Assumption: Spectral Bias.} We invoke the assumption that the effective convergence rate is frequency-dependent:
\begin{equation*}
    \lambda_k = \lambda_0 p_k^{\beta-1}, \quad \text{where } \beta > 0.
\end{equation*}
Substituting this into Equation \eqref{eq:exp_decay_raw} and defining the lumped constant $c = 2\eta' \lambda_0$:
\begin{equation*}
    \boxed{ q_k(\tau) \approx \exp\left( - c \tau p_k^\beta \right) }.
\end{equation*}
This confirms the specific kernel form for SGD discussed in Section \ref{sec:compute_scaling}.

\subsection{Part II: Proof of the Universal Scaling Theorem}
\label{app:universality}

In Section \ref{sec:general_compute_scaling}, we introduced the Universal Compute Scaling Law (Theorem \ref{thm:compute_scaling}). Here, we provide the rigorous proof showing that the scaling exponent $s = \frac{\alpha-1}{\alpha\beta}$ is a topological invariant of the effective frontier, independent of the specific choice of the optimization kernel $g(\cdot)$, provided it satisfies basic integrability conditions.

\subsubsection{Integral Approximation of the Loss}

We start with the generalized decomposition ansatz. Let the residual of pattern $k$ evolve according to a monotonic scaling kernel $g$:
\begin{equation*}
    q_k(\tau) \approx g\left( a \tau p_k^\beta \right),
\end{equation*}
where $p_k = Z^{-1} k^{-\alpha}$. The total reducible loss is approximated by the integral over the rank space $k \in [1, \infty)$:
\begin{equation*}
    \Delta L(\tau) \approx \int_{1}^{\infty} p_k \, g(a \tau p_k^\beta) \, dk = \int_{1}^{\infty} \frac{1}{Z} k^{-\alpha} g\left( a \tau (Z^{-1} k^{-\alpha})^\beta \right) dk.
\end{equation*}
Let $A = a Z^{-\beta}$ be the lumped rate constant. The argument of the kernel is $u(k) = A \tau k^{-\alpha\beta}$.

\subsubsection{Change of Variables (The Geometric Invariance)}

To separate the time scale $\tau$ from the geometric structure, we perform the change of variables $u = A \tau k^{-\alpha\beta}$.
Solving for $k$, we have:
\begin{equation*}
    k = (A \tau)^{\frac{1}{\alpha\beta}} u^{-\frac{1}{\alpha\beta}}.
\end{equation*}
The differential transforms as:
\begin{equation*}
    dk = -\frac{1}{\alpha\beta} (A \tau)^{\frac{1}{\alpha\beta}} u^{-\frac{1}{\alpha\beta}-1} du.
\end{equation*}
The integration limits transform as follows:
\begin{itemize}
    \item Lower limit $k=1 \implies u_{max} = A \tau$.
    \item Upper limit $k \to \infty \implies u_{min} \to 0$.
\end{itemize}
Substituting these into the integral:
\begin{align*}
    \Delta L(\tau) &\approx \frac{1}{Z} \int_{A\tau}^{0} \underbrace{\left[ (A \tau)^{\frac{1}{\alpha\beta}} u^{-\frac{1}{\alpha\beta}} \right]^{-\alpha}}_{k^{-\alpha}} \cdot g(u) \cdot \underbrace{\left( -\frac{1}{\alpha\beta} (A \tau)^{\frac{1}{\alpha\beta}} u^{-\frac{1}{\alpha\beta}-1} \right) du}_{dk} \\
    &= \frac{1}{\alpha\beta Z} \int_{0}^{A\tau} \left[ (A \tau)^{-\frac{\alpha}{\alpha\beta}} u^{\frac{\alpha}{\alpha\beta}} \right] g(u) \left[ (A \tau)^{\frac{1}{\alpha\beta}} u^{-\frac{1}{\alpha\beta}-1} \right] du.
\end{align*}
We now collect the powers of the time-dependent term $(A \tau)$ and the integration variable $u$.

\textbf{1. The Scaling Exponent (Time Dependence):}
The total exponent for $\tau$ is strictly determined by the Zipfian tail $\alpha$ and the optimizer bias $\beta$:
\begin{equation*}
    -\frac{\alpha}{\alpha\beta} + \frac{1}{\alpha\beta} = -\frac{\alpha-1}{\alpha\beta} := -s.
\end{equation*}

\textbf{2. The Mellin Kernel (Coefficient):}
The exponent for the integration variable $u$ is:
\begin{equation*}
    \frac{\alpha}{\alpha\beta} - \frac{1}{\alpha\beta} - 1 = \frac{\alpha-1}{\alpha\beta} - 1 = s - 1.
\end{equation*}

Substituting these back, the expression factorizes perfectly:
\begin{equation*}
    \Delta L(\tau) \approx \left[ \frac{A^{-s}}{\alpha\beta Z} \right] \cdot \tau^{-s} \cdot \int_{0}^{A\tau} u^{s-1} g(u) du.
\end{equation*}

\subsubsection{Convergence Conditions (Replacing the "Three Pillars")}

As $\tau \to \infty$, the upper limit $A\tau \to \infty$. For the scaling law to be valid (i.e., for the loss to be finite and well-behaved), the integral term must converge to a constant. This integral is the \textbf{Mellin Transform} of the kernel $g$, denoted $\mathcal{M}[g](s)$.

The condition $\mathcal{M}[g](s) < \infty$ implies two physical constraints on the learning dynamics, which replace the heuristic arguments often used in literature:

\begin{enumerate}
    \item \textbf{Integrability at the Tail ($u \to 0$):} This regime corresponds to $k \to \infty$ (rare patterns). Since $g(0)=1$ (unlearned), the integrand behaves as $u^{s-1}$. Convergence requires $s > 0$, which implies $\alpha > 1$. This formally proves that scaling laws are only possible for heavy-tailed distributions with finite mean.
    \item \textbf{Decay at the Head ($u \to \infty$):} This regime corresponds to $k \to 1$ (frequent patterns). We require $g(u)$ to decay faster than $u^{-s}$ for large $u$. This physically means the optimizer must be capable of effectively minimizing the loss for frequent patterns.
\end{enumerate}

\subsubsection{Final Result}

Assuming these conditions hold, we recover the Universal Compute Scaling Law stated in Theorem \ref{thm:compute_scaling}:
\begin{equation*}
    \Delta L(\tau) \sim \mathcal{K} \cdot \tau^{-s},
\end{equation*}
where the pre-factor $\mathcal{K}$ encapsulates all algorithmic details via the Mellin Transform:
\begin{equation*}
    \mathcal{K} = \frac{1}{\alpha\beta} \left(\frac{1}{Z}\right)^{1/\alpha} a^{-s} \mathcal{M}[g](s).
\end{equation*}
This concludes the proof.

\section{Proofs and Derivations for Joint Scaling}
\label{app:bottleneck_proofs}

In this appendix, we provide the rigorous mathematical proofs for the Composition Law (Proposition \ref{prop:max_bottleneck}) and the detailed algebraic derivations for the optimal scaling frontiers discussed in Section \ref{sec:joint_scaling}.

\subsection{Proof of the Max-Bottleneck Principle}

We wish to show that the joint reducible loss $\Delta L(N,D,\tau)$ is asymptotically dominated by the maximum of the individual resource bottlenecks:
\begin{equation*}
    \Delta L(N,D,\tau) \asymp \max\left( \varepsilon_N(N), \; \varepsilon_D(D), \; \varepsilon_\tau(\tau) \right).
\end{equation*}

\begin{proof}
\textbf{Lower Bound:}
Since model capacity $N$, dataset size $D$, and training compute $\tau$ impose distinct, necessary constraints on the pattern rank space, the joint loss cannot be lower than the limit imposed by any single bottleneck. Formally:
\begin{equation*}
    \Delta L \ge \varepsilon_N(N), \quad \Delta L \ge \varepsilon_D(D), \quad \Delta L \ge \varepsilon_\tau(\tau).
\end{equation*}
It follows directly that:
\begin{equation*}
    \Delta L \ge \max\left( \varepsilon_N(N), \; \varepsilon_D(D), \; \varepsilon_\tau(\tau) \right).
\end{equation*}

\textbf{Upper Bound:}
Consider the standard inequality for non-negative numbers $x,y,z$: $\max(x,y,z) \le x+y+z \le 3\max(x,y,z)$.
We assume a constructive upper bound where the total residual is at most the sum of residuals from each mechanism (e.g., a pattern is considered unlearned if it fails \textit{any} of the criteria: insufficient capacity, insufficient coverage, or insufficient optimization):
\begin{equation*}
    \Delta L \lesssim \varepsilon_N(N) + \varepsilon_D(D) + \varepsilon_\tau(\tau).
\end{equation*}
Applying the max-sum inequality:
\begin{equation*}
    \varepsilon_N + \varepsilon_D + \varepsilon_\tau \le 3 \max\left( \varepsilon_N, \varepsilon_D, \varepsilon_\tau \right).
\end{equation*}
Since constant factors (like 3) do not affect asymptotic scaling laws (power-law exponents), we conclude:
\begin{equation*}
    \Delta L \asymp \max\left( \varepsilon_N(N), \; \varepsilon_D(D), \; \varepsilon_\tau(\tau) \right).
\end{equation*}
\end{proof}

\subsection{The Optimal Compute Point (Turnover Analysis)}
\label{app:turnover_analysis}

For a fixed model configuration defined by parameters $N$ and dataset $D$, the learning curve is a function of training steps $\tau$. We define the \textbf{Static Bottleneck} as the limit imposed by the fixed resources:
\begin{equation*}
    \varepsilon_{\text{stat}}(N,D) = \max\left( \varepsilon_N(N), \; \varepsilon_D(D) \right).
\end{equation*}

\begin{definition}[Optimal Compute Point]
The Optimal Compute Point $\tau_\star$ is defined as the training step count where the dynamic convergence bottleneck intersects the static bottleneck. This represents the transition from the optimization-limited regime to the capacity/data-limited regime.
\end{definition}

Mathematically, this occurs when:
\begin{equation*}
    \varepsilon_\tau(\tau_\star) \asymp \varepsilon_{\text{stat}}(N,D).
\end{equation*}
Substituting the power law form $\varepsilon_\tau(\tau) = G \tau^{-\alpha_\tau}$ derived in the main text:
\begin{align*}
    G \tau_\star^{-\alpha_\tau} &\asymp \varepsilon_{\text{stat}} \\
    \tau_\star^{-\alpha_\tau} &\asymp \frac{\varepsilon_{\text{stat}}}{G} \\
    \tau_\star &\asymp \left( \frac{\varepsilon_{\text{stat}}}{G} \right)^{-1/\alpha_\tau}.
\end{align*}
The compute budget required to reach this point (FLOPs) is given by $C_\star \approx 6 N \tau_\star$.

\subsection{Derivation of Optimal Scaling Frontiers}
\label{app:optimal_derivations}

We now derive the optimal scaling laws by solving the constrained optimization problem for the joint loss function $\Delta L$. We analyze two distinct regimes corresponding to the assumptions made by \citet{kaplan2020scalinglawsneurallanguage} and \citet{hoffmann2022trainingcomputeoptimallargelanguage}.

\subsubsection{Regime A: Convergence-Limited (Kaplan Scaling)}
\label{subapp:kaplan_derivation}

We assume the "Data-Abundant" limit where the dataset size $D$ is effectively infinite (or repeated without penalty~\citep{yan2025largerdatasetsrepeatedmore}), rendering $\varepsilon_D$ negligible. The optimization problem trades off Model Size $N$ against Training Steps $\tau$.

\begin{enumerate}
    \item \textbf{Setup:} Minimize the loss subject to a fixed compute budget $C$ (FLOPs).
    \begin{itemize}
        \item Objective: $\min_{N, \tau} \max\left( A N^{-\alpha_N}, \; G \tau^{-\alpha_\tau} \right)$
        \item Constraint: $C \approx 6 N \tau \implies \tau \approx \frac{C}{6N}$.
    \end{itemize}

    \item \textbf{Substitution:} Substituting $\tau$ into the objective function yields a function of $N$ alone:
    \begin{equation*}
        \Psi(N) = \max\left( A N^{-\alpha_N}, \; G \left( \frac{C}{6N} \right)^{-\alpha_\tau} \right).
    \end{equation*}
    Absorbing constants into $G'$, this simplifies to:
    \begin{equation*}
        \Psi(N) \asymp \max\left( A N^{-\alpha_N}, \; G' C^{-\alpha_\tau} N^{\alpha_\tau} \right).
    \end{equation*}

    \item \textbf{Equilibrium:} The term $N^{-\alpha_N}$ is strictly decreasing in $N$, while $N^{\alpha_\tau}$ is strictly increasing. The minimum of the max function occurs exactly at the intersection where the two bottlenecks are balanced:
    \begin{equation*}
        A N^{-\alpha_N} \asymp G' C^{-\alpha_\tau} N^{\alpha_\tau}.
    \end{equation*}

    \item \textbf{Solution:} Solving for $N$:
    \begin{align*}
        N^{-\alpha_N} \cdot N^{-\alpha_\tau} &\asymp \frac{G'}{A} C^{-\alpha_\tau} \\
        N^{-(\alpha_N + \alpha_\tau)} &\asymp \text{const} \cdot C^{-\alpha_\tau} \\
        N &\propto C^{\frac{\alpha_\tau}{\alpha_N + \alpha_\tau}}.
    \end{align*}
\end{enumerate}
This recovers the scaling law form proposed by \citet{kaplan2020scalinglawsneurallanguage}.

\begin{remark}[Loss Scaling]
Substituting $N_{opt}$ back into the dominant term $N^{-\alpha_N}$ yields the optimal loss scaling:
\begin{equation*}
    \Delta L_{opt}(C) \asymp \left( C^{\frac{\alpha_\tau}{\alpha_N + \alpha_\tau}} \right)^{-\alpha_N} = C^{-\frac{\alpha_N \alpha_\tau}{\alpha_N + \alpha_\tau}}.
\end{equation*}
\end{remark}

\subsubsection{Regime B: Data-Limited (Chinchilla Scaling)}
\label{subapp:chinchilla_derivation}

We assume the "Data-Limited" regime where we optimize for the lowest loss on a held-out distribution where unique tokens are the limiting factor. Here, $\tau$ is assumed sufficient to cover $D$ (i.e., one epoch), so the trade-off is between Model Size $N$ and Dataset Size $D$.

\begin{enumerate}
    \item \textbf{Setup:} Minimize the loss subject to a fixed compute budget $C$ (FLOPs).
    \begin{itemize}
        \item Objective: $\min_{N, D} \max\left( A N^{-\alpha_N}, \; B D^{-\alpha_D} \right)$
        \item Constraint: $C \approx 6 N D \implies D \approx \frac{C}{6N}$.
    \end{itemize}

    \item \textbf{Substitution:} Substituting $D$ into the objective function:
    \begin{equation*}
        \Phi(N) = \max\left( A N^{-\alpha_N}, \; B \left( \frac{C}{6N} \right)^{-\alpha_D} \right).
    \end{equation*}
    Absorbing constants into $B'$:
    \begin{equation*}
        \Phi(N) \asymp \max\left( A N^{-\alpha_N}, \; B' C^{-\alpha_D} N^{\alpha_D} \right).
    \end{equation*}

    \item \textbf{Equilibrium:} The minimum occurs at the intersection of the capacity bottleneck and the coverage bottleneck:
    \begin{equation*}
        A N^{-\alpha_N} \asymp B' C^{-\alpha_D} N^{\alpha_D}.
    \end{equation*}

    \item \textbf{Solution:} Solving for $N$:
    \begin{align*}
        N^{-\alpha_N} \cdot N^{-\alpha_D} &\asymp \frac{B'}{A} C^{-\alpha_D} \\
        N^{-(\alpha_N + \alpha_D)} &\asymp \text{const} \cdot C^{-\alpha_D} \\
        N &\propto C^{\frac{\alpha_D}{\alpha_N + \alpha_D}}.
    \end{align*}
\end{enumerate}
This recovers the compute-optimal scaling law proposed by \citet{hoffmann2022trainingcomputeoptimallargelanguage}.

\section{Related Work}
\label{sec:related_work}

Neural scaling laws describe the robust power-law improvements in test loss as model size, dataset size, and compute grow~\citep{hestness2017deeplearningscalingpredictable,kaplan2020scalinglawsneurallanguage}. While test loss typically improves smoothly, \citet{wei2022emergentabilitieslargelanguage} observe that specific downstream capabilities often manifest as sharp, emergent phase transitions at scale. A central practical challenge lies in defining compute-optimal frontiers: while \citet{kaplan2020scalinglawsneurallanguage} prescribed prioritizing parameter scaling, \citet{hoffmann2022trainingcomputeoptimallargelanguage} (Chinchilla) later demonstrated that proportional data scaling is required for optimality. While these empirical laws provide crucial guidelines, the theoretical mechanism that reconciles these seemingly contradictory trade-offs remains under-explored in a unified framework.

To explain these empirical phenomena, theoretical explanations have largely relied on solvable surrogates and spectral analysis. \citet{Bahri_2024}, \citet{maloney2022solvablemodelneuralscaling}, and \citet{bordelon2021spectrumdependentlearningcurves} link scaling exponents to the spectral decay in kernel and random-feature limits, identifying a ``spectral principle'' where models learn successive frequency modes of the target function. Similarly, \citet{Spigler_2020} relate the power-law exponent directly to the decay rate of the target function's projection coefficients onto the kernel's eigenbasis. \citet{hutter2021learningcurvetheory,pan2025understandingllmbehaviorscompression} explores minimal theoretical models where power-law learning curves arise directly from the data distribution. \citet{paquette202543phasescomputeoptimalneural} and \citet{lin2025scalinglawslinearregression} further analyze compute-optimality in linear settings, showing how implicit regularization and covariance spectra drive power-law behavior. Alternatively, \citet{sharma2020neuralscalinglawdimension} connect parameter scaling to the intrinsic dimension of the data manifold. A generalized theory that captures the heavy-tailed, discrete nature of learning tasks, without being bound to specific model classes, remains necessary.

Beyond static spectral limits, recent work also derives scaling directly from optimization dynamics and discrete structures. Analyses of SGD on teacher--student models~\citep{ren2025emergencescalinglawssgd,arous2025learningquadraticneuralnetworks} and structured features~\citep{bordelon2022learningcurvessgdstructured,bordelon2024dynamicalmodelneuralscaling} demonstrate how optimization biases and feature learning~\citep{bordelon2025featurelearningimproveneural} shape learning curves over time. In the high-precision regime, \citet{Michaud_2023} highlight that neural networks scale by auto-discovering modular structures, a property distinct from classical approximation. Closest to our conceptual framework, \citet{michaud2024quantizationmodelneuralscaling} propose a ``quantization model'' where discrete skills are learned by frequency. These discrete perspectives, however, have yet to formally unify capacity, coverage, and optimization bottlenecks.

\input{appendix/compute_assump}

\end{document}

%% file: appendix/compute_assump.tex
\section{Theoretical Justifications for Lemma~\ref{lemma:q_dynamic} and Assumption~\ref{assump:spectral_power}}
\label{app:derivation_assumptions}
In this section, we provide a theoretical justification for the dynamics of residual $q_k$ in Lemma~\ref{lemma:q_dynamic} and the power law of the correction coefficient $\lambda_k$ in Assumption~\ref{assump:spectral_power}. Specifically, we utilize the Gradient Flow framework on Deep Linear Networks (DLN) to derive these properties.

\subsection{Problem Setup}
\paragraph{Model Architecture and Effective Weight.} 
We consider an $L$-layer deep linear neural network. To facilitate theoretical analysis, we assume the network operates on the eigenbasis of the data, where the input features are decoupled, similarly to Assumption~\ref{assump:decomposition}.
Let $x_k$ denote the input for the $k$-th pattern. The network's prediction $\hat{z}_k$ is obtained by passing $x_k$ sequentially through $L$ layers of weights:
\begin{equation*}
    \hat{z}_k = w_{L,k}(t) \cdot w_{L-1,k}(t) \cdots w_{1,k}(t) \cdot x_k,
\end{equation*}
where $w_{l,k}(t)$ represents the weight of layer $l$ for pattern $k$ at time $t$. 
We define the effective weight $u_k(t)$ as the product of weights across all layers:
\begin{equation*}
    u_k(t) \triangleq \prod_{l=1}^L w_{l,k}(t).
\end{equation*}
Consequently, the network function simplifies to a linear scaling in the effective parameter space: $\hat{z}_k = u_k(t) \cdot x_k$.

\textit{Note.}\quad The over-parameterization via depth $L$ introduces non-convex optimization dynamics with respect to the individual weights $w_{l,k}$, although we consider a linear mapping from $x_k$ to $\hat{z}_k$. This Deep Linear Network (DLN) formulation is a canonical model for studying feature learning~\citep{saxe2013exact}, which captures the essential stage-like learning transitions and depth-induced acceleration observed in complex models (like Transformers), while remaining analytically tractable.

\paragraph{Loss Function.} 
Assume the target output $z_k^*$ is generated by an underlying optimal feature strength $u_k^*$, such that $z_k^* = u_k^* \cdot x_k$. Adopting the Mean Squared Error (MSE) loss, the objective function $L_k$ for the $k$-th pattern is derived by taking the expectation over the data distribution:
\begin{align*}
    L_k &= \mathbb{E}_{x} \left[ \frac{1}{2} (\hat{z}_k - z_k^*)^2 \right]  \\
        &= \mathbb{E}_{x} \left[ \frac{1}{2} (u_k \cdot x_k - u_k^* \cdot x_k)^2 \right]  \\
        &= \frac{1}{2} (u_k - u_k^*)^2 \cdot \mathbb{E}_{x} [x_k^2]  \\
        &= \frac{1}{2} p_k (u_k - u_k^*)^2. 
\end{align*}
Without loss of generality, we consider that: when pattern $k$ is sampled, its input feature has unit magnitude ($x_k^2 = 1$); otherwise $x_k = 0$. 
Thus $\mathbb{E}[x_k^2]$ corresponds strictly to the data frequency of the $k$-th pattern, \textit{i.e.}, $p_k = \mathbb{E}[x_k^2]$.

The total loss is summed over all patterns:
\begin{equation}\label{eq:mse_loss}
    L = \frac{1}{2} \sum_{k}  p_k (u_k - u_k^*)^2
\end{equation}
We now explicitly map this loss structure (specific MSE loss with $L$-layer linear network) to the Atomic Pattern Decomposition in Assumption~\ref{assump:decomposition}.
Specifically, comparing with Equation~\eqref{eq:loss_decomposition}, we can identify that: the squared error term $\frac{1}{2}(u_k - u_k^*)^2$ in Equation~\eqref{eq:mse_loss} captures the unlearnedness of the pattern $k$, and the the normalized squared error can be viewed as the normalized residual $q_k$ in the main text,
\begin{equation*}
    q_k(t) \triangleq \frac{(u_k(t) - u_k^*)^2}{(u_k^*)^2} \propto (u_k(t) - u_k^*)^2.
\end{equation*}
Under this mapping, $q_k=1$ at initialization with $u_k(0) \approx 0$ and $q_k \to 0$ upon convergence with $u_k(t) \to u_k^*$, which is fully consistent with the loss form in the main text.
Thus, in the following, we focus on analyzing the convergence dynamics of the error term $(u_k(t) - u_k^*)^2$, which reflects the decay of the residual $q_k$.

\begin{assumption}[\textbf{Learnability Condition}]\label{assump:learnability}
    With a constant $\zeta > 0$, we assume that the optimal feature strength $u_k^*$ follows a power law with data frequency:
    \begin{equation*}
        u_k^* \propto p_k^\zeta.
    \end{equation*}
\end{assumption}
\textbf{Understanding of Assumption \ref{assump:learnability}.}\quad
The parameter $\zeta$ controls the richness of the optimal feature strength relative to the input distribution.
In natural language, the input spectrum $p_k$ follows a heavy-tailed Zipfian distribution. Assumption~\ref{assump:learnability} posits that the semantic importance of a pattern is naturally aligned with its frequency with $\zeta > 0$. In other words, frequent patterns (e.g., core syntactic structures or common topic words) generally possess larger ground-truth weights ($u_k^*$).
If $\zeta \le 0$, it means that the model requires to learn large target weights ($u_k^*$) from rare patterns (small $p_k$). However, the gradient descent update step is proportional to the data frequency (Equation~\eqref{eq:ode_general}), the effective learning rate vanishes due to the tiny $p_k$. This renders the large target $u_k^*$ practically unlearnable as the optimization stalls in prolonged plateaus.


\subsection{Gradient Flow Dynamics}

We analyze the learning dynamics under continuous-time Gradient Descent (Gradient Flow, $\eta \to 0$). Applying the chain rule, the time evolution of the effective weight $u_k$ is:
\begin{align*}
    \frac{d u_k}{dt} &= \sum_{l=1}^L \frac{\partial u_k}{\partial w_{l,k}} \frac{d w_{l,k}}{dt} \nonumber \\
    &= \sum_{l=1}^L \frac{\partial u_k}{\partial w_{l,k}} \left( - \eta \frac{\partial \mathcal{L}}{\partial u_k} \frac{\partial u_k}{\partial w_{l,k}} \right) \nonumber \\
    &= - \eta \frac{\partial \mathcal{L}}{\partial u_k} \sum_{l=1}^L \left( \frac{\partial u_k}{\partial w_{l,k}} \right)^2.
\end{align*}
Under small balanced initialization, gradient flow in deep linear networks maintains the invariant $\frac{d}{dt}(w_{l,k}^2 - w_{j,k}^2) \to 0$, implying equal weight magnitudes across layers: $|w_{l,k}| \approx |u_k|^{1/L}$~\citep{arora2018optimization}. 
Using the relation $\frac{\partial u_k}{\partial w_{l,k}} = \frac{u_k}{w_{l,k}}$, we can simplify the preconditioner term:
\begin{equation*}
    \sum_{l=1}^L \left( \frac{\partial u_k}{\partial w_{l,k}} \right)^2 = \sum_{l=1}^L \left( \frac{u_k}{w_{l,k}} \right)^2 = \sum_{l=1}^L \frac{u_k^2}{|u_k|^{2/L}} = L \cdot |u_k|^{2 - \frac{2}{L}}.
\end{equation*}
Substituting the gradient $\frac{\partial \mathcal{L}}{\partial u_k} = p_k(u_k - u_k^*)$ into the dynamics equation yields the non-linear Ordinary Differential Equation (ODE) governing $u_k$,
\begin{equation}
    \label{eq:ode_general}
    \frac{d u_k}{dt} = \eta L p_k u_k^{2 - \frac{2}{L}} (u_k^* - u_k).
\end{equation}

\subsection{Convergence Analysis and Effective Rate}
We now derive the asymptotic convergence rate for both shallow ($L=2$) and general deep networks.

\paragraph{Case 1: Shallow Networks ($L=2$).}
Assuming $u_k^* > 0$ and $u_k(0) > 0$ without loss of generality, Equation~\eqref{eq:ode_general} simplifies to the standard Logistic differential equation:
\begin{equation*}
    \frac{d u_k}{dt} = 2 \eta p_k u_k (u_k^* - u_k).
\end{equation*}
The analytical solution is
\begin{equation*}
    u_k(t) = \frac{u_k^* \cdot A e^{2 \eta \Lambda_k t}}{1 + A e^{2 \eta \Lambda_k t}} = \frac{u_k^* \cdot e^{2 \eta \Lambda_k t}}{e^{2 \eta \Lambda_k t} + A^{-1}},
\end{equation*}
where $\Lambda_k = p_k u_k^*$ and $A = \frac{u_k(0)}{u_k^* - u_k(0)}$ is determined by initialization.
In the convergence phase (large $t$), the error term decays exponentially:
\begin{align*}
    u_k^* - u_k(t) = \frac{u_k^* A^{-1}}{e^{2 \eta \Lambda_k t} + A^{-1}} \approx u_k^* A^{-1} e^{-2 \eta \Lambda_k t}
    \propto e^{-2 \eta \Lambda_k t},
\end{align*}
and thus
\begin{align}\label{eq:error_decay_L_2}
    (u_k(t)-u_k^*)^2 \propto \exp \left(-4 \eta \underbrace{p_k u_k^*}_{\Lambda_k} t\right).
\end{align}

\paragraph{Case 2: General Deep Networks ($L > 2$).}
For general depth, the exact solution is implicit. Separating variables in Equation~\eqref{eq:ode_general} yields:
\begin{equation*}
    t = \frac{1}{\eta L p_k} \int_{u_k(0)}^{u_k(t)} \frac{dz}{z^{2 - \frac{2}{L}} (u_k^* - z)}.
\end{equation*}
Substituting $y = z / u_k^*$, the integral becomes:
\begin{equation*}
    t = \frac{1}{\eta L p_k (u_k^*)^{2 - \frac{2}{L}}} \int_{y_0}^{y(t)} \frac{dy}{y^{2 - \frac{2}{L}} (1 - y)}.
\end{equation*}
The integral involves the Gaussian hypergeometric function. However, we are primarily interested in the asymptotic behavior in the convergence regime ($y \to 1$). Expanding the integrand $\frac{1}{y^{2 - 2/L} (1-y)}$ near $y=1$, the dominant term is $\frac{1}{1-y}$.
Performing the integration:
\begin{equation*}
    \int \frac{dy}{1-y} = -\ln(1-y) = -\ln \frac{u_k^* - u_k}{u_k^*}.
\end{equation*}
Thus, the time-evolution is asymptotically governed by:
\begin{equation*}
    t \approx \frac{1}{\eta L p_k (u_k^*)^{2 - \frac{2}{L}}} \left[ -\ln(u_k^* - u_k) + C \right].
\end{equation*}
Rearranging for the error term, we obtain the decay law:
\begin{equation}
    \label{eq:error_decay}
    (u_k(t) - u_k^*)^2 \propto \exp \left( - 2\eta L \underbrace{\left[ p_k (u_k^*)^{2 - \frac{2}{L}} \right]}_{\Lambda_k} t \right).
\end{equation}
This confirms that for any depth $L$, the asymptotic convergence rate is governed by $\Lambda_k = p_k (u_k^*)^{2 - \frac{2}{L}}$, which is consistent with the analytical solution derived for the $L=2$ case.

\subsection{Justification of Lemma~\ref{lemma:q_dynamic} and Assumption~\ref{assump:spectral_power}}
Finally, we bridge the derived convergence to Lemma~\ref{lemma:q_dynamic} and Assumption~\ref{assump:spectral_power}.
Recall that the normalized residual $q_k(t)$ represents the squared error component. Based on the parameter error decay in Equation~\eqref{eq:error_decay}:
\begin{equation*}
    q_k(t) \propto (u_k(t) - u_k^*)^2 \propto  \exp\left(- 2 \eta L \Lambda_k t\right).
\end{equation*}
Discretizing this continuous decay with step size $\Delta t = 1$, the update rule becomes:
\begin{equation*}
    q_k(t+1) = q_k(t) \cdot e^{- 2 \eta L \Lambda_k}.
\end{equation*}
Applying the first-order Taylor expansion $e^{-x} \approx 1 - x$ for small learning rates $\eta$, we have
\begin{equation}\label{eq:q_k_Lambda}
    q_k(t+1) \approx (1 - \eta \cdot [2 L \Lambda_k]) q_k(t),
\end{equation}
which further verifies Lemma~\ref{lemma:q_dynamic}.

Comparing Equation~\eqref{eq:q_k_Lambda} with $q_k(t+1) = (1 - \eta \lambda_k) q_k(t)$ in Lemma~\ref{lemma:q_dynamic}, we identify the per-step correction coefficient as:
\begin{equation*}
    \lambda_k \equiv 2 L \Lambda_k.
\end{equation*}
With $\Lambda_k = p_k (u_k^*)^{2 - \frac{2}{L}}$ and Assumption~\ref{assump:learnability}:
\begin{align*}
    \lambda_k &\propto L p_k (u_k^*)^{2 - \frac{2}{L}}  \\
              &\propto L p_k (p_k^\zeta)^{2 - \frac{2}{L}} \\
              &= L p_k^{1 + \zeta(2 - \frac{2}{L})},
\end{align*}
which provides a theoretical justification for Assumption~\ref{assump:spectral_power} under SGD dynamics in $L$-layer deep linear networks.

In Assumption~\ref{assump:spectral_power}, we posit that $\lambda_k \propto p_k^{\beta - 1}$. Matching the exponents, we theoretically determine the scaling coefficient $\beta$:
\begin{equation}
    \beta - 1 = 1 + \zeta \left( 2 - \frac{2}{L} \right) \implies \beta = 2 + \zeta \left( 2 - \frac{2}{L} \right).
\end{equation}
Since $L \ge 1$ and $\zeta > 0$, the scaling exponent satisfies $\beta \ge 2$, which places the optimization in the Rich Regime (characterized by $\beta > 1$). 

The derived expression $\beta = 2 + \zeta(2 - 2/L)$ mathematically formalizes the interaction between the data structure and model architecture.
Specifically, the depth-dependent term $2 - 2/L$ introduces a non-linear selectivity mechanism: for dominant features with large target magnitudes of $u_k^*$, the network provides super-linear acceleration; conversely, for weak or noisy features with small magnitudes of $u_k^*$, the learning dynamics are severely suppressed. This mechanism reflects the \textit{implicit bias} of deep networks, driving the optimization to prioritize high-frequency strong features over weak patterns.